%% file: neurips_2025.tex
\documentclass{article}
\PassOptionsToPackage{numbers,compress}{natbib}

\usepackage[preprint]{neurips_2025}

\usepackage[utf8]{inputenc} 
\usepackage[T1]{fontenc}    
\usepackage{hyperref}       
\usepackage{url}            
\usepackage{booktabs}       
\usepackage{amsfonts}       
\usepackage{nicefrac}       
\usepackage{microtype}      
\usepackage{xcolor}         

\usepackage{enumitem}
\usepackage{graphicx}
\usepackage{custom_def}

\title{A Dual Basis Approach for\\ Structured Robust Euclidean Distance Geometry}

\author{%
  Chandra Kundu \\
  Department of Statistics and Data Science \\
  University of Central Florida \\
  Orlando, FL 32816, USA \\
  \texttt{chandra.kundu@ucf.edu}\\
  \AND
  Abiy Tasissa\\
  Department of Mathematics\\
  Tufts University\\
  Medford, MA 02155, USA\\
  \texttt{abiy.tasissa@tufts.edu}\\
  \And
  HanQin Cai\\
  Department of Statistics and Data Science\\
  Department of Computer Science\\
  University of Central Florida\\
  Orlando, FL 32816, USA\\
  \texttt{hqcai@ucf.edu}\\
}

\begin{document}

\maketitle

\begin{abstract}
  \textit{Euclidean Distance Matrix} (EDM), which consists of pairwise squared Euclidean distances of a given point configuration, finds many applications in modern machine learning. This paper considers the setting where only a set of anchor nodes is used to collect the distances between themselves and the rest. In the presence of potential outliers, it results in a structured partial observation on EDM with partial corruptions. Note that an EDM can be connected to a positive semi-definite Gram matrix via a non-orthogonal dual basis. Inspired by recent development of non-orthogonal dual basis in optimization, we propose a novel algorithmic framework, dubbed \textit{Robust Euclidean Distance Geometry via Dual Basis} (RoDEoDB), for recovering the Euclidean distance geometry, i.e., the underlying point configuration. The exact recovery guarantees have been established in terms of both the Gram matrix and point configuration, under some mild conditions. Empirical experiments show superior performance of RoDEoDB on sensor localization and molecular conformation datasets. 
\end{abstract}

\vspace{-0.12in}
\section{Introduction}
\vspace{-0.15in}
Understanding the similarity and dissimilarity among a set of objects is a fundamental task across many areas of applied sciences, including machine learning \cite{tenenbaum2000global}, sensor networks localization \cite{ding2010sensor,biswas2006semidefinite} and computational chemistry \cite{glunt1993molecular,fang2013using}. A common quantitative approach to capturing this information is through distance matrices. Given a distance matrix, one often seeks a configuration of points in low-dimensional space (e.g., $d = 2$ or $d = 3$) that exactly or approximately realizes the given distances. This facilitates both visualization and further analysis \cite{einav2023quantitatively,yang2018learned}.

This paper focuses on (squared) \textit{Euclidean Distance Matrices} (EDMs), which consist of pairwise squared Euclidean distances of given point configurations. Working with EDMs presents several challenges. First, some distances may be missing due to measurement limitations. Second, even when the matrix is fully observed, it may contain significant noise or outliers due to faulty sensors or environmental factors. In such cases, standard algorithms designed for exact or mildly noisy distance matrices can perform poorly, often yielding low-quality embeddings. We note that the problem of reconstructing the configuration of points from missing distances is known as the \textit{Euclidean Distance Geometry} (EDG) problem \cite{liberti2014euclidean,dokmanic2015euclidean}. 

A fundamental problem, then, is to robustly recover the underlying structure by identifying and mitigating the effect of outliers. One approach is inspired by Robust Principal Component Analysis (RPCA), which decomposes a matrix into a low-rank component and a sparse outlier component \cite{candes2011robust}. For distance matrices, this idea is appealing because the EDM of $T$ points in $d$-dimensional space has rank at most $d+2$ \cite{dokmanic2015euclidean}. However, a major limitation of this approach is that the resulting low-rank matrix may not satisfy the structural properties of an EDM such as non-negativity and the triangle inequality, especially under large corruption. To address this, previous work on distance matrix completion has advocated working instead with the Gram matrix, which captures inner products of the point cloud and is symmetric and positive semidefinite \cite{Javanmard_2012,abiy_exact, Li2024, ghosh2024, Smith2024}. Importantly, if one recovers a valid low-rank Gram matrix, the corresponding EDM can be reconstructed. Furthermore, since the Gram matrix of a generic $d$-dimensional point set has rank $d$, which is typically lower than the rank $d+2$ of EDM, this formulation potentially allows for greater robustness to outliers.

However, three central challenges remain. Firstly, in practice, we cannot observe Gram matrix entries directly. Thus, any optimization involving the Gram matrix must rely on a linear map relating the Gram matrix to the EDM. The linear map is formed on a non-orthogonal dual basis, and thus this approach is known as the dual basis approach. It has been successfully employed in prior work on EDM completion \cite{abiy_exact}. Secondly, large-scale distance matrices present computational and memory challenges \cite{arefin2012computing}. 
Thirdly, in many applications, such as sensor network localization, only partial and noisy distance measurements are available, which are typically between target nodes (with unknown positions) and anchor nodes (with known positions) \cite{dargie2010fundamentals}. Distances between target nodes may be missing or impractical to acquire due to cost, energy or physical constraints \cite{kuriakose2014review}.

To address these challenges, we propose a robust estimation algorithm that uses a structure inspired by Nystr\"om approximation \cite{platt2004nystrom, NIPS2000_19de10ad}, focusing on the recovery of the underlying point configuration from partially corrupted distance measurements between anchors and targets. Although our formulation is  motivated by sensor localization, the underlying model applies to a variety of domains, such as molecular conformation, where such structured partial observations might arise. A key technical challenge is that the observed entries, in the form of a submatrix of the distance matrix, do not directly correspond to a submatrix of the same dimensions in the Gram matrix. Thus, a careful formulation is required to operate on only the relevant blocks of the Gram matrix. To the best of our knowledge, this is the first work to tackle the robust EDG problem using a dual basis approach combined with structured anchor-target measurements.

\noindent\textbf{Related work.} 
Classical EDG methods aim to reconstruct point configurations from fully or partially observed pairwise distance measurements. Traditional solutions include multidimensional scaling-based formulations \cite{young1938discussion,torgerson1952multidimensional,torgerson1958theory,gower1966some,de2009multidimensional}, which typically assume complete and noise-free EDMs. Although effective under ideal conditions, their accuracy degrades rapidly with missing or corrupted data \cite{cayton2006robust}. To handle corrupted measurements, several robust EDG methods have been introduced \cite{cayton2006robust,zhang2019localization,liu2018robust,forero2012sparsity,zhou2011closed,deng2025robust}. These approaches generally employ robust loss functions, convex relaxations, or RPCA to separate sparse noise from the low-rank structures. However, these methods often require observing full EDM, limiting their applicability when substantial portions, particularly target-target distances, are missing.

An alternative formulation, the anchor-target sampling model, addresses scenarios where distances are primarily available between a small subset of anchor nodes and a larger set of target nodes \cite{fang1986trilateration,thomas2005revisiting,wang2015linear}. Nystr\"om-based techniques are commonly used in this context, approximating the full Gram matrix from observed anchor-anchor and anchor-target blocks \cite{platt2004nystrom}. When the rank of anchor-anchor Gram block is the same as the embedding dimension, the unobserved target-target block can be accurately recovered via Nystr\"om extension \cite{kumar2009sampling}. These methods are widely adopted in kernel approximation tasks
\cite{gittens2013revisiting,kumar2012sampling,NIPS2000_19de10ad}. Such anchor-target structures could be particularly relevant in applications like sensor networks and molecular conformation, where acquiring comprehensive target-target distances is impractical or prohibitively expensive. Recently, the dual basis framework was developed to address anchor-based EDG with non-orthogonal measurement models \cite{lichtenberg2024localization}. While dual basis approaches offer elegant theoretical foundations, existing literature primarily assumes ideal noise-free conditions. To address robustness, recent work introduced a two-step approach: First, denoising anchor-target distances via RPCA, then reconstructing the Gram matrix using Nystr\"om extensions \cite{kundu2025structured}. However, this does not provide recovery guarantees under sparse outlier corruption.


Alternating projection methods have established their effectiveness for robust matrix recovery \cite{wang2013robust, jiang2017robust, escalante2011alternating,waldspurger2018phase,cai2019accelerated, gu2016low, netrapalli2014non, tanner2023compressed}. These methods iteratively project estimates between constraint sets, such as data-consistent measurements and low-rank matrices, efficiently handling incomplete and corrupted observations. Our approach builds on this insight by integrating the dual basis formulation with an alternating projection strategy that operates iteratively projecting onto the space of distances consistent with observed anchor-target measurements and the space of low-rank Gram structures. This enables accurate recovery of the point configuration while effectively mitigating the impact of outliers.

\noindent\textbf{Main contributions.} 
The main contributions of this paper are three-fold:
\begin{itemize}[leftmargin=*]
    \vspace{-0.1in}
    \item We introduce \textit{RoDEoDB}, a robust and effective approach for EDG with corrupted pairwise distances observed only between anchor and target nodes. The algorithm directly operates on a block of Gram matrix, leveraging its low-rank structure and avoiding the need for full EDM observations.
    \vspace{-0.02in}
    \item Under incoherence conditions on EDM and sparsity on outliers, we provide a theoretical analysis establishing exact recovery guarantees for both the Gram matrix and the point configuration.
    \vspace{-0.02in}
    \item Through synthetic experiments and real protein structure prediction tasks using data from the Protein Data Bank \cite{berman2000protein}, we demonstrate that the proposed algorithm effectively removes or minimizes the influence of outliers in distance measurements and obtains high-quality embeddings.
\end{itemize}




\vspace{-0.12in}
\section{Preliminaries} \label{sec:Preliminaries}
\vspace{-0.15in}

\noindent\textbf{Notation.} 
Vectors are denoted by bold lowercase symbols such as $\u$, matrices by bold uppercase symbols like $\M$, and operators or sets by calligraphic letters such as $\mathcal{H}$. For a vector $\u$, $[\u]_i$ denotes the i-th entry of $\u$;  $\|\u\|_0$ denotes the number of nonzero entries, while $\|\u\|_2$ represents its $\ell_2$-norm. For a matrix $\M$, the $(i,j)$-th entry is written as $[\M]_{i,j}$, the transpose is denoted by $\M^\top$, and the Moore–Penrose pseudoinverse by $\M^\dagger$. $\m_i$ denotes i-th column of the matrix $\M$ and $\M_{i:}$ denotes the i-th row of $\M$. The $i$-th largest singular value of $\M$ is denoted by $\sigma_i(\M)$; the spectral norm is $\|\M\|_2 = \sigma_1(\M)$; and the condition number is defined as $\kappa(\M) = \frac{\sigma_1(\M)}{\sigma_{r}(\M)}$, where $r$ is the rank of $\M$. We use $\1_m$ to denote the column vector of ones of length $m$, and $\I$ to denote the identity matrix. The $i$-th standard basis vector is written as $\boldsymbol{e}_i$. The notation $\mathbb{E}([\M]_{i,j})$ refers to the average of all elements in $\M$. We write $[T]$ for the index set $\{1, 2, \dots, T\}$, and for any subset $\mathcal{I}$, the expression $|\mathcal{I}|$ denotes its cardinality. Superscripts in parentheses, such as $M^{(k)}$ or $\xi^{(k)}$, indicate the value of a variable at the $k$-th iteration of an algorithm.

\noindent\textbf{Structured measurements.} 
We consider a set of $T$ points in Euclidean space $\mathbb{R}^{d}$ divided into two groups: $m$ anchor nodes $\P_{1} = [\p_1, \cdots, \p_m] \in \mathbb{R}^{d \times m}$ and $n$ target nodes $\P_{2} = [\p_{m+1}, \cdots, \p_{m+n}] \in \mathbb{R}^{d \times n}$. The complete set of points is denoted by $\P = [\P_{1}, \P_{2}] \in \mathbb{R}^{d \times T}$, where $T = m + n$. The corresponding EDM, denoted by $\D \in \mathbb{R}^{T \times T}$, is symmetric with its entries defined as $[\D]_{ij} = \|\p_i - \p_j\|_2^2$, capturing the pairwise squared pairwise distances. We partition $\D$ as follows:
\begin{equation}\label{eq:EDM_blocks}
  \D =:
  \begin{bmatrix}
    \E & \F \\
    \F^\top & \G
  \end{bmatrix},
\end{equation}
where $\E \in \mathbb{R}^{m \times m}$ denotes the sub-EDM of the anchor nodes, $\F \in \mathbb{R}^{m \times n}$ denotes the pairwise squared distances between anchors and targets, and $\G \in \mathbb{R}^{n \times n}$ is the sub-EDM of the target nodes. In our setting, $\E$ is exactly given as the locations of the anchors are known. However, $\F$ may be corrupted by sparse outliers due to, for example, sensor malfunction. The target-target sub-EDM $\G$ is entirely missing or unavailable, which is a common scenario in many applications. Furthermore, we assume that the last row of $\F$ is exact, corresponding to the central node or base station, as commonly considered in the sensor localization literature. This reflects the practical advantage of central nodes, which often possess superior hardware, stable power sources, and access to external references (e.g., GPS), leading to more accurate or trusted measurements \cite{McGrath2013,yick2008wireless}. Extending the formulation to account for uncertainty in the measurements corresponding to a central is a promising direction for future work. Our objective is to robustly recover the entire point configuration $\P$ solely from the exact distances in $\E$ and the partially corrupted distances $\F$, without any information from the $\G$ block.

\noindent\textbf{Anchor-centric operators.} 
To exploit structural properties beneficial for robust EDM recovery, we adopt a Gram-matrix-based anchor-centric formulation. Specifically, we consider the Gram matrix $\X = \P^\top \P$, partitioned compatibly with $\D$ into blocks:
\begin{equation}\label{eq:Gram_blocks}
  \X = \begin{bmatrix}
    \P_{1}^\top \P_{1} & \P_{1}^\top \P_{2} \\
    \P_{2}^\top \P_{1} & \P_{2}^\top \P_{2}
  \end{bmatrix} =: 
  \begin{bmatrix}
    \A & \B \\
    \B^\top & \C
  \end{bmatrix}. 
\end{equation}

The EDM $\D$ and Gram matrix $\X$ can be connected through a double-centering operation \citep{gower1982euclidean,gower1985properties}. Given a vector $\s \in \mathbb{R}^T$ satisfying $\s^\top \1 = 1$, a Gram matrix can be recovered from the EDM by:
\begin{align}
\label{eq:double_centering}
  \X &= -\frac{1}{2}(\I - \1\s^\top)\D(\I - \s\1^\top).
\end{align}
Obviously, if the center of the point configuration is not fixed, then the Gram matrix is not unique. 
In our anchor-centric setting, we choose a specialized $\s = \frac{1}{m}[\underbrace{1,\dots,1}_{m},\underbrace{0,\dots,0}_{n}]^\top \in \mathbb{R}^{T}$, which will place the centroid of the anchor points at the origin in the recovered Gram matrix, i.e., $\P_1 \1_{m}=\bm{0}$. Thus, we have a one-to-one correspondence between EDM and Gram matrix. Expanding \eqref{eq:double_centering} yields:
\begin{align}
  \A &= -\frac{1}{2}\left(\E -\frac{1}{m}\1_{m}\1_{m}^\top \E - \frac{1}{m}\E\1_{m}\1_{m}^\top + \mathbb{E}([\E]_{i,j}) \1_{m}\1_{m}^\top\right), \label{eq:AfromE}\\
  \B &= -\frac{1}{2}\left(\F-\frac{1}{m}\1_{m}\1_{m}^{\top}\F-\frac{1}{m}\E\1_{m}\1_{n}^{\top}+\mathbb{E}([\E]_{i,j})\1_{m}\1_{n}^{\top}\right). \label{eq:BfromEF}
\end{align}
This anchor-centric formulation is particularly advantageous because it provides a direct and geometrically meaningful mapping from the distance blocks $\E$ and $\F$ to the corresponding blocks $\A$ and $\B$ in the Gram matrix, without any dependency on the unknown/missing target-target distance block $\G$.

Furthermore,  \cite[Theorem 2]{lichtenberg2024localization} establishes an explicit relationship between the Gram block $\B$ and the anchor-target distance block $\F$ under this anchor-centric geometry. The result shows:
\begin{align*}
  \F = \1_{m} \F_{k,:} - 2(\B - \1_{m}\B_{k,:}) + \frac{1}{m} \left( \E \1_{m} - (\E \1_{m})_{k} \1_{m} \right) \1_{n}^\top.
\end{align*}
This equation provides a direct formulation to recover the anchor-target distance block $\F$ using Gram block $\B$, the exact anchor sub-EDM $\E$, and just any one complete row from $\F$, denoted by $\F_{k,:}$.

Motivated by this formulation, we define two linear operators central to our proposed method. The first operator, $\mathcal{A}:\mathbb{R}^{m\times n}\rightarrow \mathbb{R}^{m\times n}$, maps a Gram block $\B$ to the anchor-target distance block $\F$:
\begin{align}\label{eq:operator_A}
  \mathcal{A}(\B) = \1_{m} \F_{k,:} - 2(\B - \1_{m}\B_{k,:}) 
  + \frac{1}{m}\left(\E \1_{m} - (\E \1_{m})_{k}\1_{m}\right)\1_{n}^\top.
\end{align}
Conversely, the operator $\mathcal{B}:\mathbb{R}^{m\times n}\rightarrow \mathbb{R}^{m\times n}$ maps a given anchor-target distance block $\F$ back to its corresponding Gram block $\B$:
\begin{align}\label{eq:operator_B}
  \mathcal{B}(\F) = -\frac{1}{2}\left(\F - \frac{1}{m}\1_{m}\1_{m}^{\top}\F - \frac{1}{m}\E\1_{m}\1_{n}^{\top} + \mathbb{E}([\E]_{i,j})\1_{m}\1_{n}^{\top}\right).
\end{align}
Note that the mapping $\mathcal{A}$  is constructed by using a non-orthogonal dual basis \cite{lichtenberg2024localization}, and it holds the identity $\mathcal{A}(\mathcal{B}(\F)) = \F$, given the anchors are centered at the origin.

\vspace{-0.12in}
\section{Proposed approach}
\vspace{-0.14in}

In this section, we present our dual basis approach, with theoretical guarantees, for robustly recovering a set of target points using exact anchor-anchor distances $\E$ and sparsely corrupted anchor-target distance measurements $\F$. Many prior studies \citep{deng2025robust, blouvshtein2018outlier, cayton2006robust, forero2012sparsity, li2017outlier, kong2019classical, mandanas2016robust} have studied the robust EDG problem against outliers with full distance observation. However, given hardware limitations and/or acquisition costs, the target-target sub-EDM is often large-scale, and it is frequently incomplete, massively corrupted, or even entirely unavailable \citep{yang2024wireless, wang2006coverage}. Thus, we completely avoid these target-target distances and focus solely on anchor-anchor and anchor-target distances.

The observed anchor-target distance matrix is denoted as $\F = \Fs + \Ss$, comprising a clean underlying block $\Fs$ and sparse corruption $\Ss$. A typical approach, as discussed in \cite{kundu2025structured}, involves first using RPCA to remove sparse noise from $\F$, then applying Nystr\"om approximation. However, this approach faces limitations to achieve optimal recovery performance: (i) The anchor-target distance block $\F$ inherently has a rank of up to $d+2$, whereas its corresponding Gram block $\B$ has a rank of up to $d$. Generally speaking, working with a matrix of higher rank is harder for RPCA. (ii) When directly recovering $\F$, RPCA ignores non-negativity and triangle inequality constraints that are naturally built into distances, which lowers the recoverability and robustness, as not all constraints are utilized.


We utilize two key dual basis mapping operators $\mathcal{A}$ and $\mathcal{B}$ defined in \eqref{eq:operator_A} and \eqref{eq:operator_B}, which satisfy $\mathcal{A}(\Bs) = \mathcal{A}(\mathcal{B}(\Fs)) = \Fs$ and $\F = \mathcal{A}(\Bs) + \Ss$. The proposed approach, dubbed \textit{Robust Euclidean Distance Geometry via Dual Basis} (RoDEoDB), is summarized as \Cref{algo:RoDEoDB}. RoDEoDB is composed of two sequential phases: (I) The first phase robustly recovers the Gram block $\hB$ directly from the corrupted anchor-target distance block $\F$; (II) the second phase reconstructs the full point configuration $\P$ via the Nystr\"om method using $\A$ and the recovered Gram block $\hB$.

\begin{algorithm}
  \caption{
  Robust Euclidean Distance Geometry via Dual Basis (RoDEoDB)}
  \label{algo:RoDEoDB}
  \begin{algorithmic}[1]
    \State \textbf{Input:} $\E \in \mathbb{R}^{m \times m}$: Anchor sub-EDM; $\F \in \mathbb{R}^{m \times n}$: Anchor-target distance block; $\xi_0,\gamma$: Thresholding parameters for outlier detection.
    \State Compute/define $\A, \mathcal{A}$ and $ \mathcal{B}$ using \eqref{eq:AfromE}, \eqref{eq:operator_A} and \eqref{eq:operator_B}, respectively.
    \State $\hat{\B} \leftarrow \textrm{Dual Basis Local Outlier Removal}(\F, \mathcal{A}, \mathcal{B}, \xi_0, \gamma)$, see \Cref{algo:dbrpca}. 
    \State $\hat{\C} \leftarrow \hat{\B}^{\top} \A^{\dagger} \hat{\B}$
    \State $\hat{\X} \leftarrow 
    \begin{bmatrix} \A & \hat{\B} \\
      \hat{\B}^{\top} & \hat{\C}
    \end{bmatrix}
    $
    \State $\hat{\P} \leftarrow \hat{\Sigmab}_d^{1/2} \hat{\U}_d^\top$ where $\hat{\U}_d\hat{\Sigmab}_d\hat{\U}_d^\top=\hat{\X}$
    \State \textbf{Output:} $\hat{\P}$: Recovered $d$-dimensional point configuration.
  \end{algorithmic}
\end{algorithm}

\begin{algorithm}
  \caption{Dual Basis Alternating Projections (DBAP) for Dual Basis Local Outlier Removal}
  \label{algo:dbrpca}
  \begin{algorithmic}[1]
    \State \textbf{Input:} $\F \in \mathbb{R}^{m \times n}$: Anchor-target distance block;  $\mathcal{A}$, $\mathcal{B}$: Dual basis mapping operators. $\xi_0,\gamma$: Thresholding parameters for outlier detection. 
    \State $\xi^{(0)} \leftarrow \xi_0$; \quad  $\S^{(0)} \leftarrow \mathcal{T}_{\xi_0}(\F)$; 
    \quad $\B^{(0)} \leftarrow \mathcal{H}_{d}\!\bigl(\mathcal{B}(\F - \S^{(0)})\bigr)$
    \For{$k=1,2,\cdots$}
    \State $\F^{(k)} \leftarrow \mathcal{A}\!\bigl(\B^{(k-1)}\bigr)$
    \State $\xi^{(k)} \leftarrow \gamma \xi^{(k-1)}$
    \State $\S^{(k)} \leftarrow \mathcal{T}_{\xi^{(k)}}\!\bigl(\F - \F^{(k)}\bigr)$
    \State $\B^{(k)} \leftarrow \mathcal{H}_{d} \mathcal{P}_{T^{(k)}} \!\bigl(\mathcal{B}(\F - \S^{(k)})\bigr)$
    \EndFor
    \State \textbf{Output:} $\widehat{\B}$: Recovered block in Gram matrix.
  \end{algorithmic}
\end{algorithm}

\noindent\textbf{Phase I.} 
Using $\E$ and $\F$, we first compute the Gram block $\A$ and define dual mappings $\mathcal{A}$ and $\mathcal{B}$ via \eqref{eq:AfromE}, \eqref{eq:operator_A} and \eqref{eq:operator_B}, respectively. 
It is followed by a dual basis local outlier removal procedure to recover the Gram block $\B$ directly from $\F$ without going through the full-size distance or Gram matrix. Here, inspired by the Riemannian alternating projections based RPCA algorithm \citep{cai2019accelerated}, we design our own iterative dual basis procedure, coined \textit{Dual Basis Alternating Projections} (DBAP) and summarized as \Cref{algo:dbrpca}, for local outlier removal. We should emphasize that the framework of \Cref{algo:RoDEoDB} is still valid if another dual basis outlier removal subroutine is designed and used in Line 3. 


DBAP initializes $\B$ by first removing large standout outliers in the observed $\F$, which can be detected by a hard-thresholding operator $\mathcal{T}_{\xi}(\bm{z})$, defined as:
\[
  [\mathcal{T}_{\xi}(\bm{z})]_i =
  \begin{cases}
    [\bm{z}]_i, & |[\bm{z}]_i| > \xi; \\ 0, & |[\bm{z}]_i| \leq \xi.
  \end{cases}
\]
Then the the mapping $\mathcal{B}$ take the denoised $\F-\B^{(0)}$ to the space of Gram matrix (or more precisely, the space corresponding to the anchor-target block in the Gram matrix) and finished with the best rank-$d$ approximation $\mathcal{H}_d$ since the underlying $\B$ block inherently has rank up to $d$. The best rank-$d$ approximation is achieved via truncated singular value decomposition: 
\begin{align*}
  \mathcal{H}_{d}(\M) = \sum_{i=1}^{d} \sigma_i \u_i \v_i^\top, 
\end{align*}
where $\sigma_i$, $\u_i$, and $\v_i$ are the $i$-th singular value, left and right singular vectors of $\M$, respectively.


Subsequently, DBAP iteratively refines estimates by alternating updates between the spaces of Gram and distance blocks. At each iteration, we use the mapping $\mathcal{A}$ to project onto the anchor-target distances block and refine the outlier detection with geometrically decayed thresholding values $\xi^(k) = \gamma \xi^{(k-1)}$ where $\gamma\in (0,1)$ is a decay parameter. We then use the mapping $\mathcal{B}$ to project back to the space of Gram block and enforce the low-rankness with Riemannian accelerated low-rank approximation. That is, before applying truncated SVD $\mathcal{H}_d$, it first projects the input matrix, namely $\M$, onto the tangent space of the rank-$d$ manifold at $\B^{(k)}$ as defined by:
\begin{align*}
  \mathcal{P}_{T^{(k)}}(\M):= \U^{(k)}(\U^{(k)})^\top \M + \M \V^{(k)}(\V^{(k)})^\top-\U^{(k)}(\U^{(k)})^\top \M \V^k(\V^{(k)})^\top, 
\end{align*}
where $\U^{(k)}$ and $\V^{(k)}$ are the left and right singular vectors of $\B^{(k)}$, respectively. This Riemannian accelerated low-rank approximation has been proven helpful to enhance convergence speed and reduce per-iteration computational cost via QR decompositions \citep{wei2016guarantees,cai2019accelerated}. 
This ensures $\B^{(k)}$ remains rank-$d$. The alternating projections are repeated until convergence. Two common stopping criteria can be used here are: (i) the relative error $\|\F-\F^{(k)}+\S^{(k)}\|/\|\F\|$ is small enough; or the relative change between successive estimates of the Gram block $\B$ or outlier block $\S$ is small enough.


\noindent\textbf{Phase II.} 
Once we have the estimated Gram block $\hat{\B}$, we proceed to reconstruct the underlying point configuration. The initial step is to complete the Gram matrix $\X$ following the next steps of \Cref{algo:RoDEoDB}. We compute the Gram block $\C$ using the Nystr\"om method, specifically $\hat{\C} = \hat{\B}^\top \A^\dagger \hat{\B}$. This step leverages the previously computed exact Gram block $\A$ and the estimated Gram block $\hat{\B}$ to reconstruct the full Gram matrix. 
From the recovered Gram matrix $\hat{\X}$, we can compute the point configuration $\hat{\P}$ via the (rank-$d$ truncated) eigenvalue decomposition of $\hat{\X} = \U_d \Sigmab_d \U_d^\top$, and then setting $\hat{\P} = \Sigmab_d^{1/2} \U_d^\top$. This procedure yields a $d$-dimensional Euclidean embedding of the nodes centered at the origin. The resulting embedding is unique up to rigid motions and translation. 

Note that the noiseless $\X$ is always rank-$d$ and positive semi-definite. Yet the estimated $\hat{\X}$ may carry noise, and truncated eigenvalue decomposition can help, especially when the estimated $\hat{\B}$ is less accurate. In the case when some of the top $d$ eigenvalues of $\hat{\X}$ are negative, one should also project it onto the positive semi-definite cone, i.e., set the negative eigenvalues to 0.

\noindent\textbf{Computational Complexity.}
The operators $\mathcal{A}$ and $\mathcal{B}$ each require $\mathcal{O}(mn)$ flops. At initialization, computing the initial low-rank Gram estimate of $\B$ involves a hard thresholding step followed by a partial SVD, costing $\mathcal{O}(mnd)$ flops. In subsequent iterations, the tangent space projection is carried out using several matrix multiplications, two thin QR factorizations and a small-scale SVD \citep[Section~2]{cai2019accelerated},  costing $\mathcal{O}(nd^2)$, $\mathcal{O}(md^2)$, $\mathcal{O}(d^3)$ flops, respectively. Therefore, the total per-iteration complexity is $\mathcal{O}(mnd + md^2 + nd^2 + d^3)$ flops. Since $d\ll m,n$, the overall per-iteration cost is concluded to be $\mathcal{O}(mnd)$ flops.

\noindent\textbf{Limitations.} Our  approach is designed for a specific setting of distance measurements, where certain anchors have access to pairwise distance information between themselves and all other points, even though some of these measurements may be grossly corrupted. Further, as assume the existence of a central node, from which all measurements are exact. We acknowledge that this measurement setting has its limitations and may not suit all applications of EDM. Nevertheless, we find this measurement setting favorable in many applications, such as Internet of Things and computational chemistry.

\vspace{-0.12in}
\subsection{Theoretical results}
\vspace{-0.14in}
Our analysis is based on matrix incoherence and $\alpha$-sparsity conditions, which are defined below. 
\begin{definition}[$\mu$-incoherence \cite{candes2012exact,recht2011simpler}] \label{def:incoherence}
	Let $\M \in \Real^{m \times n}$ be a matrix of rank $r$ with singular value decomposition $\M = \U_\M \Sigmab_\M \V_\M^\top$. The incoherence parameters of $\M$ are defined as:
	\begin{align*}
		\mu_1(\M)  \coloneq \frac{m}{r} \max_{1 \leq i \leq T} \|\U_\M^\top \e_i\|_2^2, \qquad
		\mu_2(\M)  \coloneq \frac{n}{r} \max_{1 \leq j \leq T} \|\V_\M^\top \e_j\|_2^2.
	\end{align*}


\end{definition}

The incoherence condition on EDM ensures that no point is isolated far away and has a much longer distance to other points. The next assumption of $\alpha$-sparsity is standard in RPCA literature.

\begin{definition}[$\alpha$-sparsty \cite{yi2016fast,cai2019accelerated}] \label{def:sparse}
	A matrix $\S \in \Real^{m \times n}$ is called $\alpha$-sparse if it holds:
	\begin{align*}
		\max_i \|\S^T\e_i\|_0 \leq \alpha n \quad \text{and} \quad \max_j \|\S\e_j\|_0 \leq \alpha m,
	\end{align*}
	This means the matrix has at most $\alpha n$ non-zero entries in each row and at most $\alpha m$ non-zero entries in each column.
\end{definition}


Now we are ready to present the main theoretical results. To ease of presentation, in the analysis, we use AltProj \cite{netrapalli2014non} as the subroutine in Line 3 of \Cref{algo:RoDEoDB}. The recovery guarantees are presented in two parallel measurements: the Gram matrix (\Cref{thm:boundX}) and the point configuration (\Cref{thm:boundP}). 
\begin{theorem} \label{thm:boundX}
  Let $\D \in \mathbb{R}^{T \times T}$ be a $\mu$-incoherent EDM, with $\rank(\D) = d+2$. Suppose that the set of anchor indices $\cI \subseteq [T]$ is uniformly sampled without replacement, and $m = |\cI|$ satisfy $m \geq \gamma (d+2) \sqrt{\frac{T}{m}} \mu(\D) \log(d)$, for some $\gamma \geq 0$. Define the set of target indices as $\cJ = [T] \setminus \cI$ with $n = |\cJ| = T - m$. 
  Define the Gram matrix $\Xs = -\frac{1}{2} \J \D \J^\top$, where $\J = \I - \1 \s^\top$, and $
		[\s]_i =
		\begin{cases}
			\frac{1}{m}, & \text{for } i \in \cI;   \\
			0,           & \text{for } i \in \cJ
		\end{cases}$.  Let $\Fs = \D(\cI, \cJ) \in \mathbb{R}^{m \times n}$ and let $\Ss \in \mathbb{R}^{m \times n}$ be $\alpha$-sparse with $\alpha \lesssim \mathcal{O} ( 1/(d+2) \mu(\D)^2 \kappa(\D)^2 )$.
  Let $\hF$ be the output obtained from running
      $\mathcal{O} \left(\frac{T}{m} \frac{\kappa(\D)}{(1-\delta)\varepsilon} \right)$
  iterations of AltProj as the subroutine in Line 3 of \Cref{algo:RoDEoDB}. Then, the estimated Gram matrix $\hX$ satisfies:
  \[
      \frac{\| \hX-\Xs\|_2}{\|\Xs\|_2} \leq \varepsilon \kappa(\D)^{-1},
  \]
with probability at least $1 -  \frac{2d}{T^{c(\delta + (1-\delta)\log(1-\delta))}}$.
\end{theorem}

\begin{theorem} \label{thm:boundP}
  Given the notations and assumptions of \Cref{thm:boundX}, let $\Ps = [\p_1, \ldots, \p_T] \in \mathbb{R}^{d \times T}$ be the underlying $d$-dimensional point configuration centered at the original. With the same probability, the estimated $\hP$ output by \Cref{algo:RoDEoDB} satisfies: 
  \[
    \min_{\Q\in\mathcal{Q}_d}\|\Ps-\Q\hat{\P}\|_2\leq \frac{5}{18}  \varepsilon \|\D\|_{2} \| {\Ps}^{\dagger} \|_{2} \, + d^{1/4} \sqrt{ \frac{5}{18} \varepsilon \|\D\|_{2} },
  \]
  where $\mathcal{Q}_d$ is the set of all orthogonal matrices of size $d \times d$ and $\Q$ is the optimal rotation alignment.
\end{theorem}

Note that our analysis assumes incoherence and randomly chosen anchors to avoid worst-case scenarios. With appropriate side information, such as structural priors or application-specific constraints, they may be relaxed. In fact, the effective placement of anchors is a trending topic to be explored.

\vspace{-0.12in}
\section{Numerical experiments}
\vspace{-0.15in}
This section verifies the numerical performance of RoDEoDB across three problems of interest: synthetic sensor localization, synthetic spiral data, and a real molecular conformation dataset. 
All experiments, including ours and competing methods, are implemented in MATLAB and conducted on a computer with Intel i9-12900H and 32GB RAM. Additional experiments are reported in \Cref{sec:more experiemtns}.

\noindent\textbf{Synthetic sensor localization.} 
We evaluate RoDEoDB in a synthetic sensor localization task, where the goal is to recover target coordinates from corrupted pairwise distances to anchor nodes. We generate $T = 500$ random points in $[-100, 100]^d$, for $d = 2,3$. The points are uniformly sampled, and then manually centered at the origin. We choose $m$ random points as anchors and the rest are targets. An $\alpha$ fraction of the entries in anchor-target distances $\Fs$ are randomly perturbed by outliers, which are uniformly drawn from $[-3\mathbb{E}(|[\Fs]_{i,j}|), 3\mathbb{E}(|[\Fs]_{i,j}|)]$.

\begin{figure}[H]
    \centering
    \begin{minipage}{0.86\textwidth}
        \centering
        \begin{subfigure}{0.32\textwidth}
            \includegraphics[width=\textwidth]{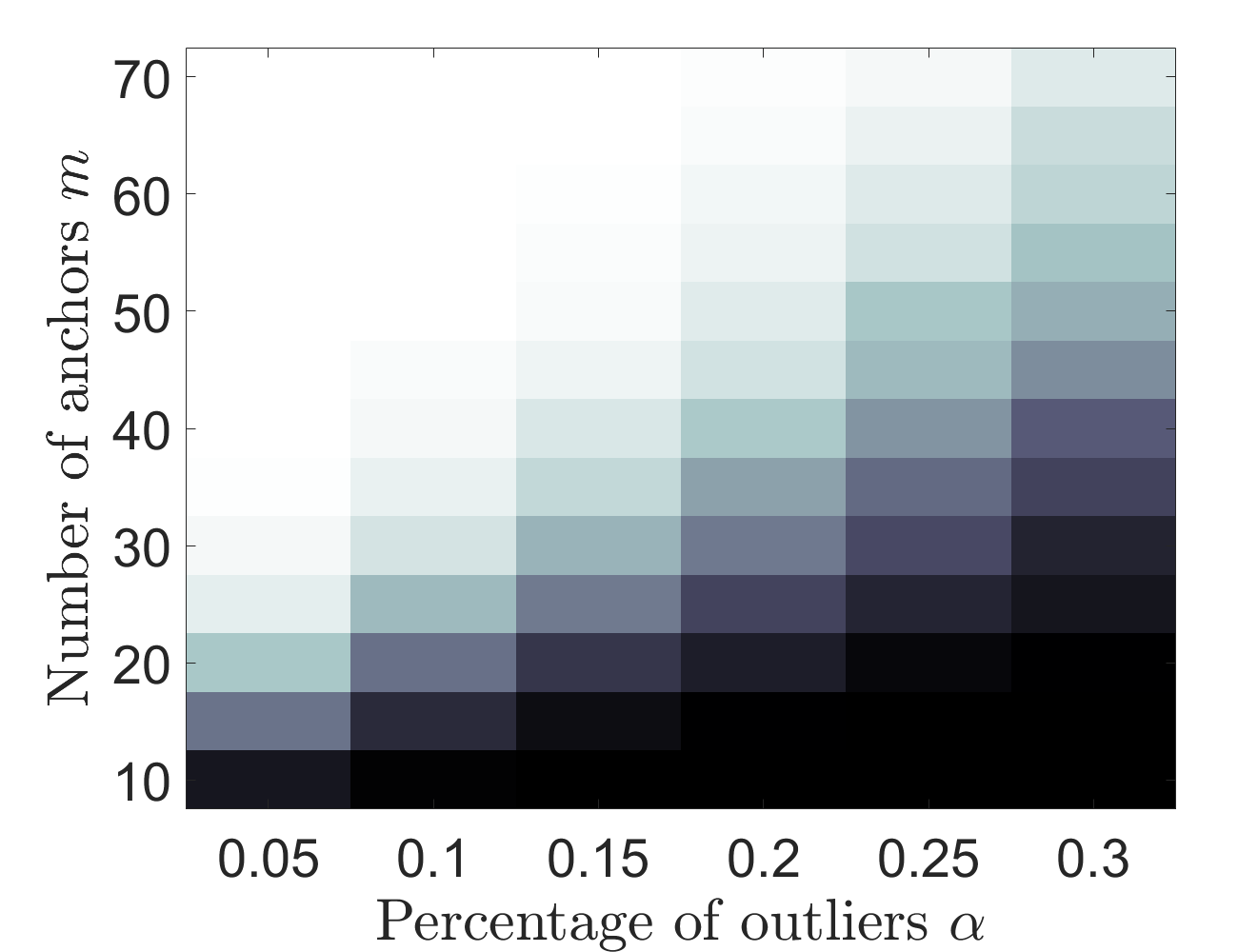}
        \end{subfigure}
        \hfill
        \begin{subfigure}{0.32\textwidth}
            \includegraphics[width=\textwidth]{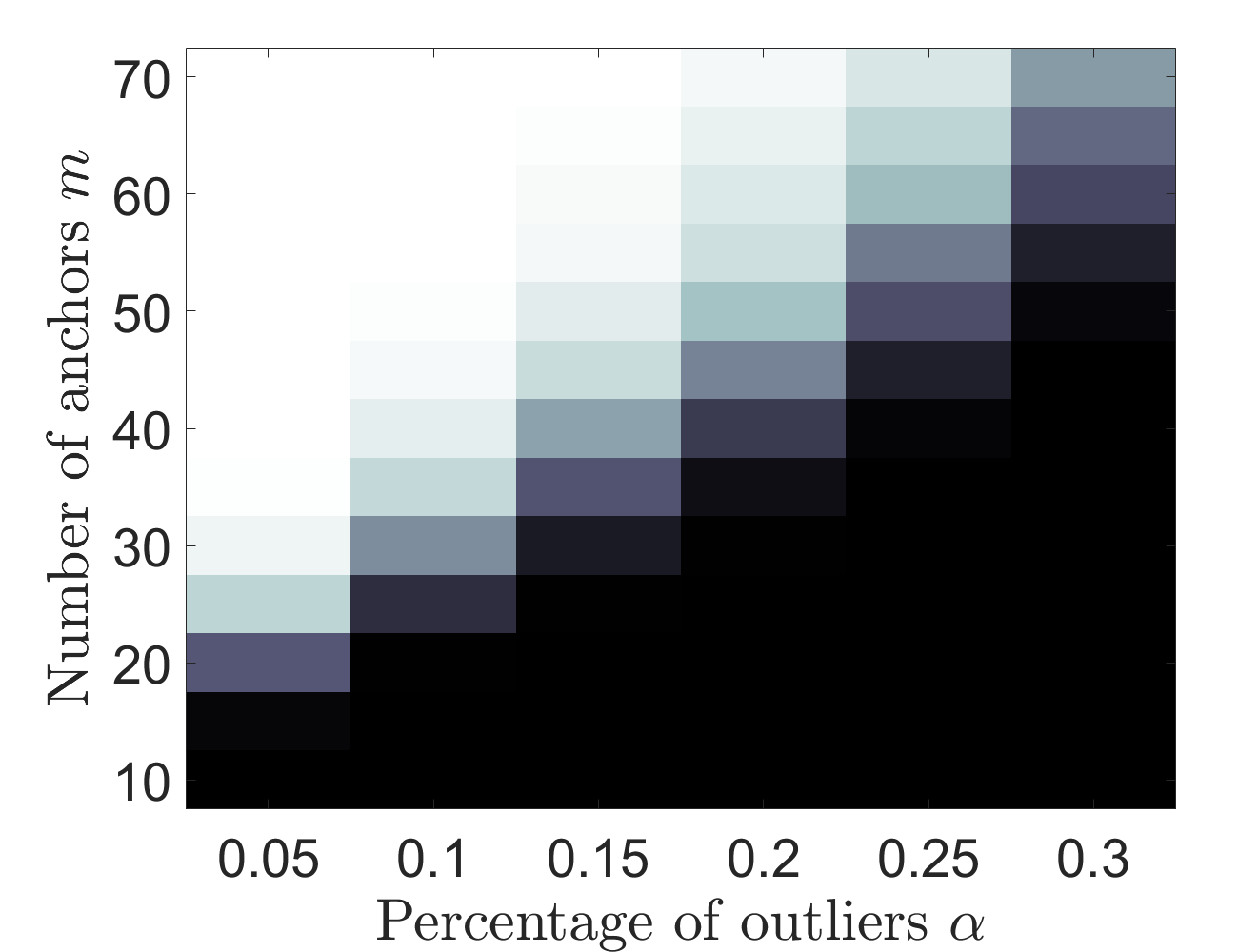}
        \end{subfigure}
        \hfill
        \begin{subfigure}{0.32\textwidth}
            \includegraphics[width=\textwidth]{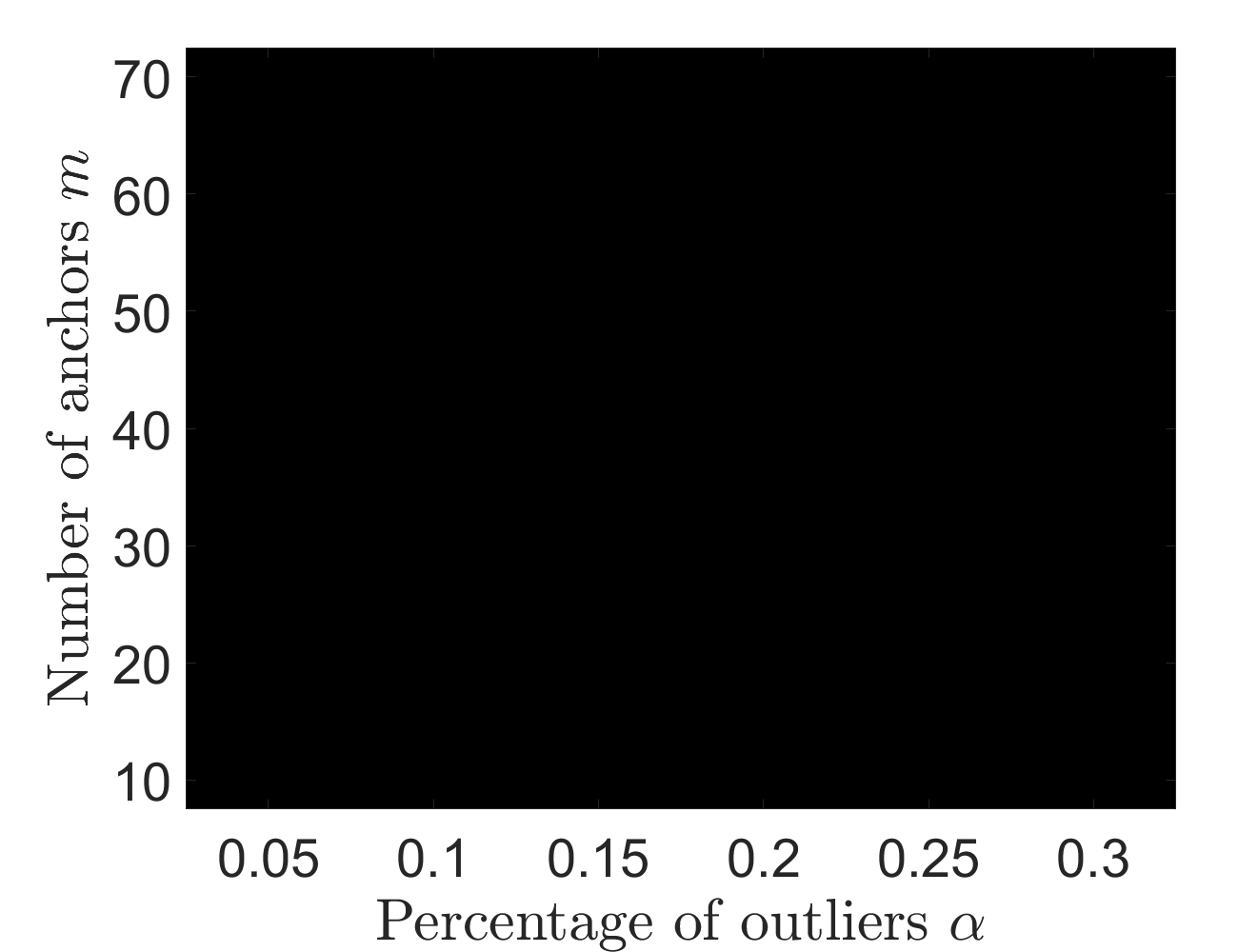}
        \end{subfigure}\\
        

        \begin{subfigure}{0.32\textwidth}
            \includegraphics[width=\textwidth]{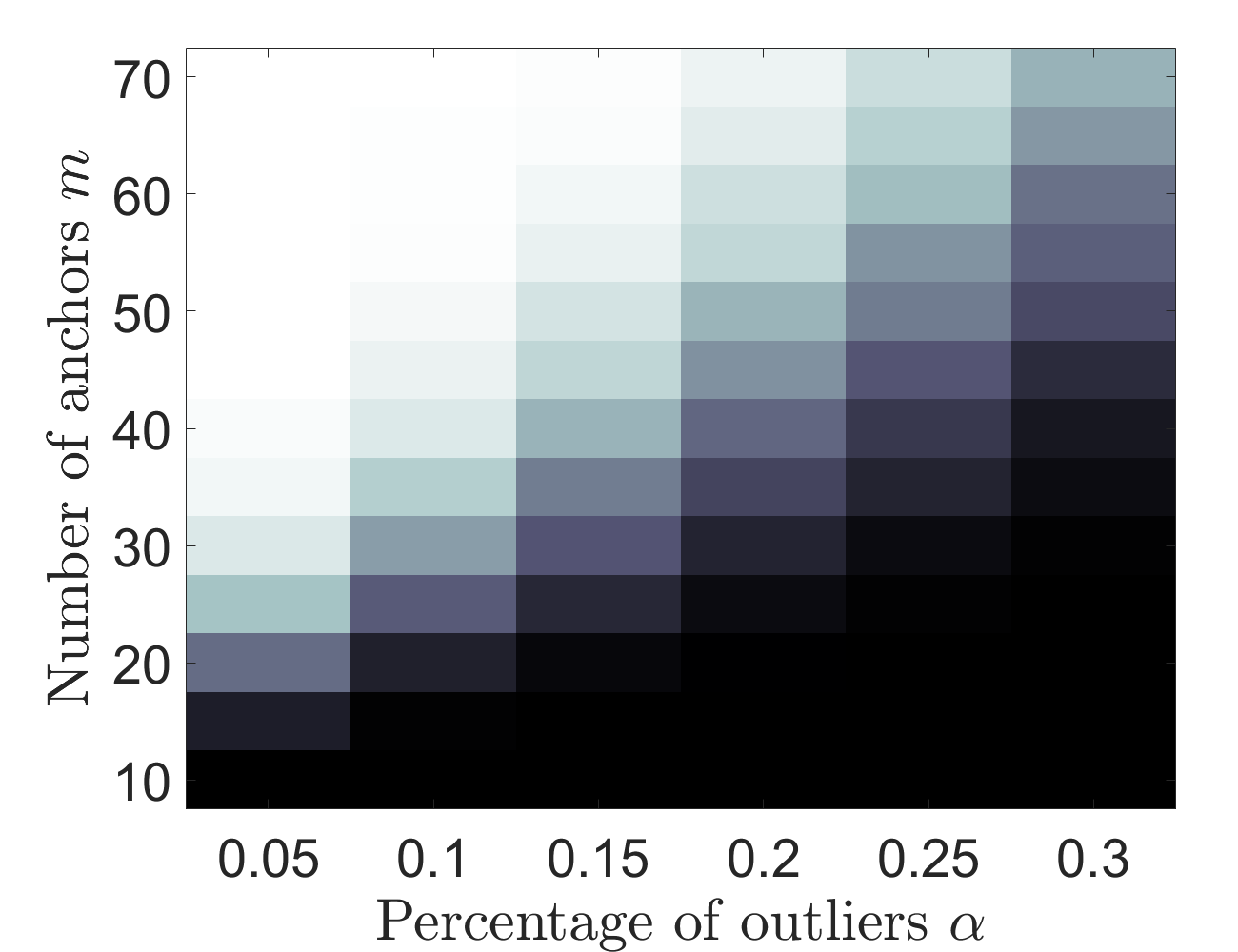}
        \end{subfigure}
        \hfill
        \begin{subfigure}{0.32\textwidth}
            \includegraphics[width=\textwidth]{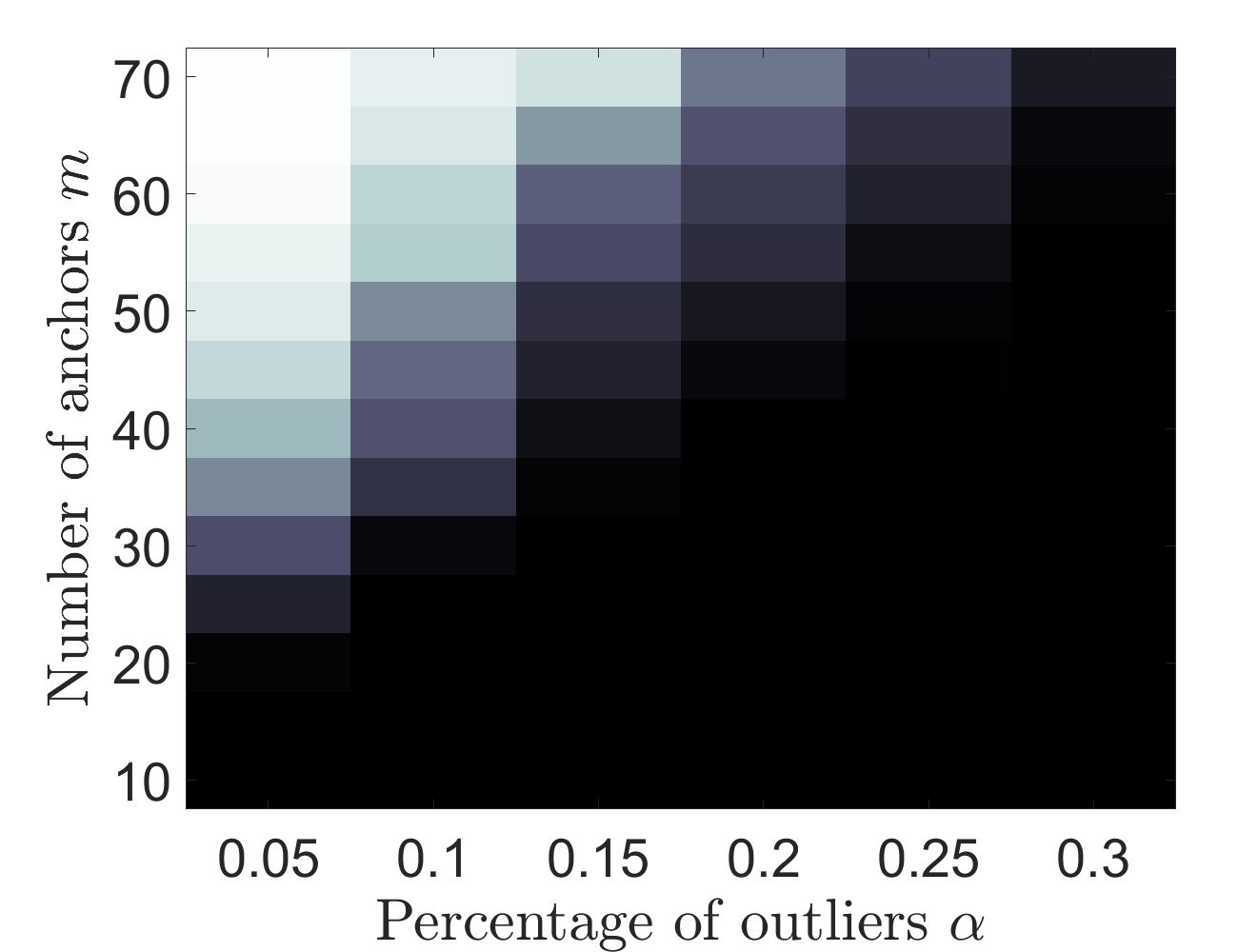}
        \end{subfigure}
        \hfill
        \begin{subfigure}{0.32\textwidth}
            \includegraphics[width=\textwidth]{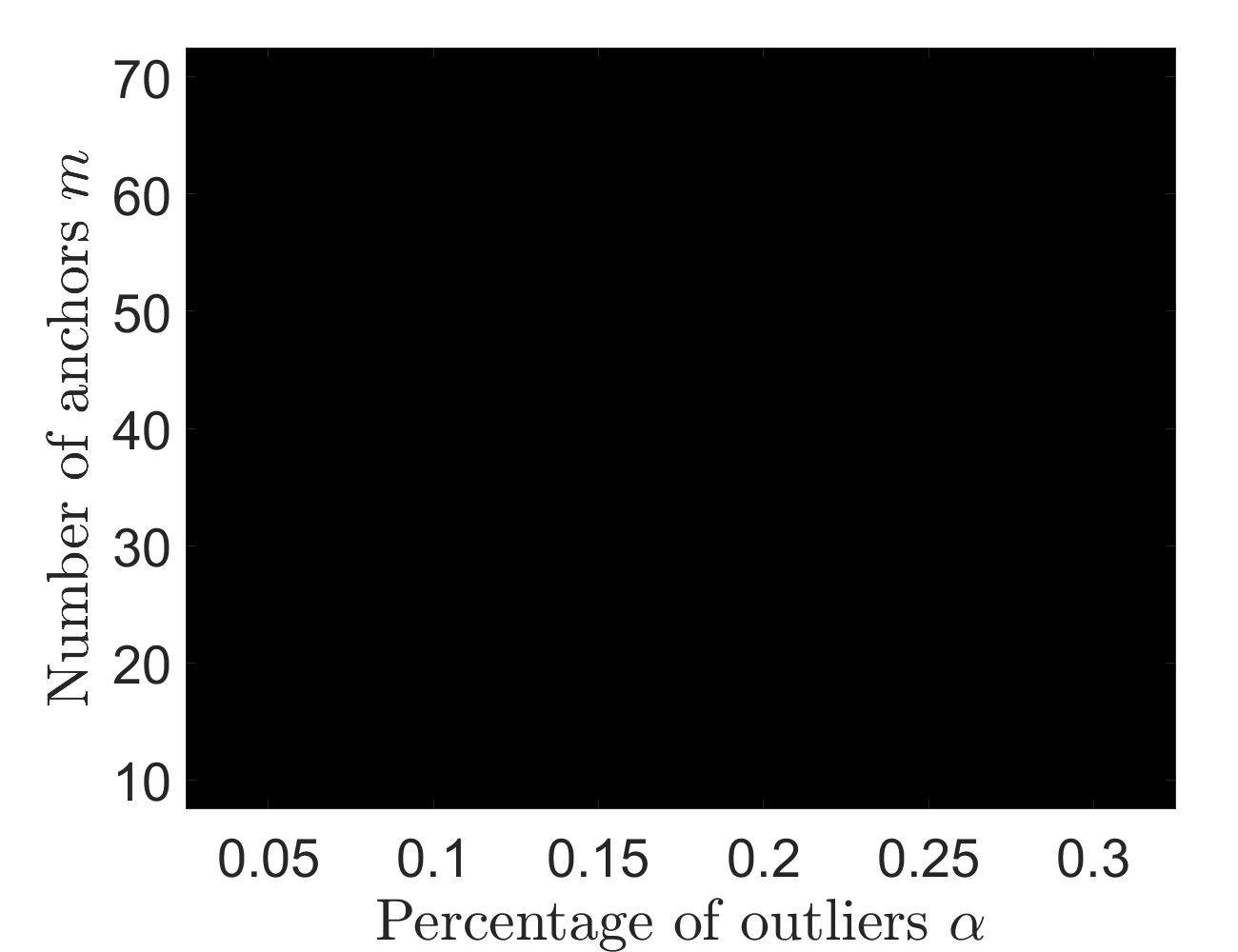}
        \end{subfigure}
    \end{minipage}%
    \begin{minipage}[c]{0.022\textwidth}
        \centering
        \includegraphics[height=5.6cm]{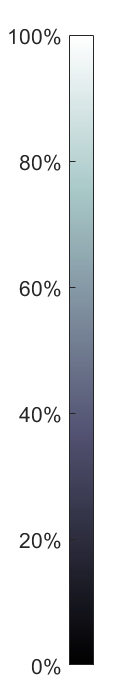}  
    \end{minipage}
    \vspace{-0.1in}
    \caption{\small Phase transition over 1000 trials showing recovery rate (RMSE $\leq 1$) for $T = 500$ sensors across varying anchor counts $m$ and outlier density $\alpha$. Top row: $d = 2$; Bottom row: $d = 3$. From left to right: RoDEoDB, SREDG, and GD. }
    \label{fig:synthetic_sensor}
    \vspace{-0.15in}
\end{figure}

We vary both the count of anchors $m \in \{10, 20, \cdots, 70\}$ and the outlier density $\alpha \in \{0.05, 0.1, \cdots, 0.3\}$. 
We evaluate the accuracy of \Cref{algo:RoDEoDB} using the root mean square error (RMSE), defined as $\text{RMSE} = {||\hat{\P}-\Ps||_\fro}/{\sqrt{T}}$, where the estimated $\hP$ is aligned with $\Ps$ via Procrustes analysis. A recovery is deemed successful if the RMSE falls below 1. We compare the proposed RoDEoDB against SREDG \cite{kundu2025structured}, which considers the same distance measurement setting. Note that the robust EDG problem can also be interpreted as robust matrix completion (RMC) directly on the EDM, as some entries of the EDM are sparsely corrupted and an entire block $\G$ is missing. Accordingly, we also include Gradient Descent (GD) \cite{yi2016fast}, a classic non-convex RMC algorithm, for comparison. The phase transition results are conducted over 1000 independent trials for each problem setting, and reported in \Cref{fig:synthetic_sensor}

When $d=2$, i.e., all points are in $\mathbb{R}^2$, RoDEoDB consistently achieves near-perfect recovery under a wide range of settings. It succeeds even with as few as 25 anchors and up to $\alpha=30\%$ outliers, indicating strong robustness to both sparse corruption and limited anchor counts. SREDG performs well under low corruption levels, but its performance deteriorates significantly when $\alpha > 20\%$ or $m < 45$. GD has theoretical guarantees under uniform sampling, yet it fails to recover any meaningful point configurations under our observation settings.

When $d=3$, the recovery task becomes slightly more difficult. RoDEoDB remains the most reliable, though its recovery region shifts slightly upward, requiring a few more anchors to tolerate the same level of corruption. SREDG exhibits similar trends but generally shows much weaker performance, especially with more outliers. GD continues to perform poorly, failing to produce meaningful reconstructions under any combinations of $m$ and $\alpha$. These results highlight the robustness and effectiveness of RoDEoDB in sensor localization problems despite significant levels of sparse corruption in the distance measurements.

\noindent\textbf{Synthetic low-dimensional spiral data embedded in high-dimensional space.} 
We setup a challenging experiment by constructing a $2D$ synthetic spiral embedded in $10D$ space with both sparse outliers and dense Gaussian noise. This setup evaluates the performance of RoDEoDB in recovering an intrinsically low-dimensional structure embedded in higher dimensions \cite{tzeng2008multidimensional}. We begin by sampling $p$ total points along a $2D$ spiral defined by the polar equation $r = 2\theta$, where $\theta$ is uniformly spaced over $[2\pi, 5\pi]$. The first two coordinates of each point are computed as
$
    x = r \cos(\theta) + \epsilon $ and $ y = r \sin(\theta) + \epsilon,
$
where $\epsilon \sim \mathcal{N}(0,1)$. 
To construct a $10D$ embedding, we append 8 additional coordinates sampled from $\mathcal{N}(0,1)$. The resulting point configuration is then manually centered at the origin. Outliers are added in the same procedure as described in the sensor localization experiment.

\Cref{fig:synthetic_spiral} illustrates the reconstruction performance. RoDEoDB successfully recovers the spiral structure even with as few as 20 anchors. As $m$ increases, the reconstruction quality improves further, closely matching the ground truth. In contrast, SREDG fails to recover the global structure for low anchor counts ($m = 20, 30$), showing significant distortion and collapsing of the spiral geometry. Only when $m = 40$, SREDG begins to approximate the spiral correctly.

\begin{figure}[ht]
\vspace{-0.05in}
  \centering
  \begin{subfigure}[b]{0.38\textwidth}
    \centering
    \includegraphics[%
      width=\linewidth,
      height=0.9\textheight,      
      keepaspectratio
    ]{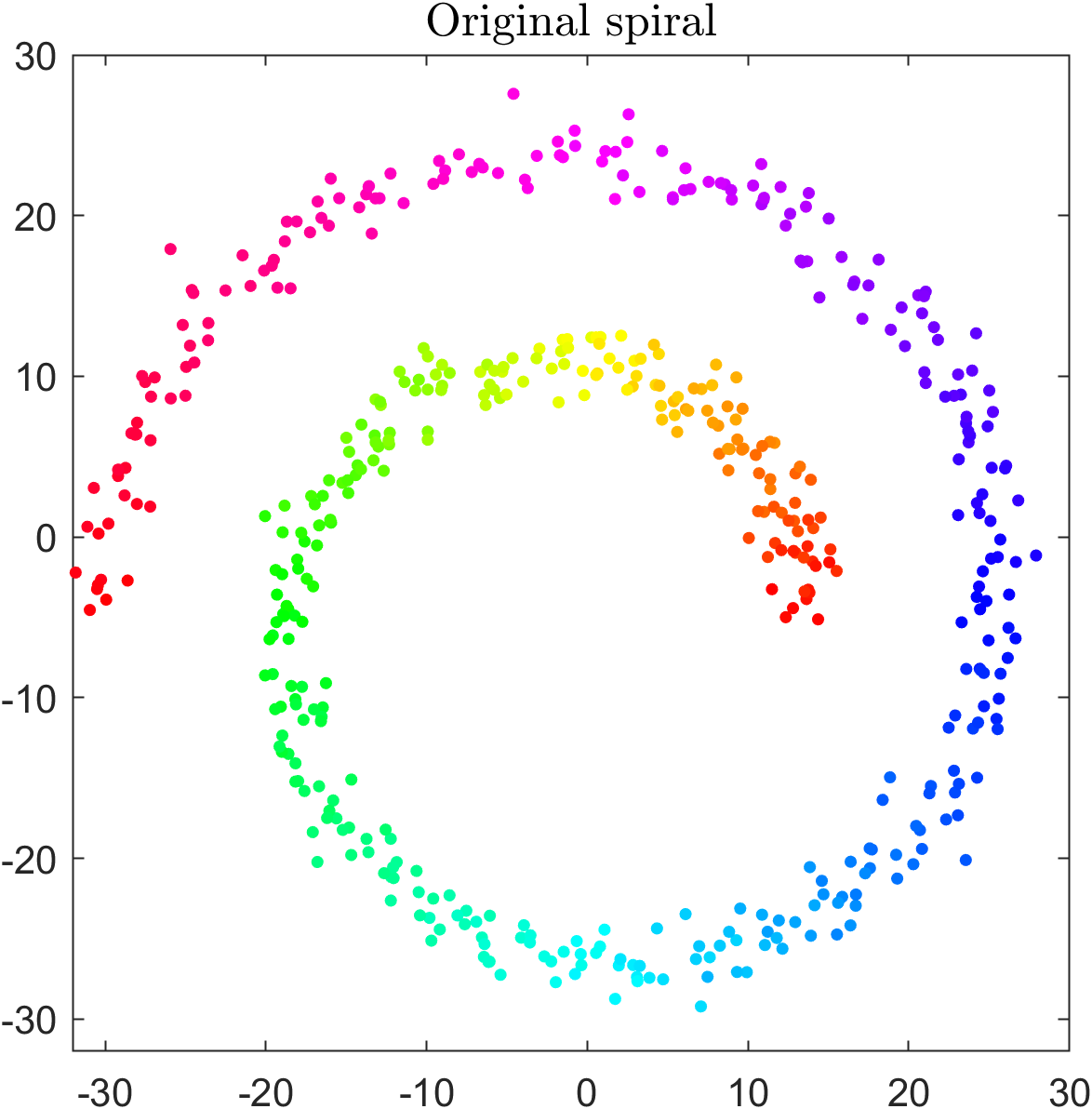}
  \end{subfigure}%
  \hfill
  \begin{minipage}[b]{0.56\textwidth}
    \centering
    \begin{subfigure}[b]{0.32\linewidth}
      \centering
      \includegraphics[width=\linewidth]{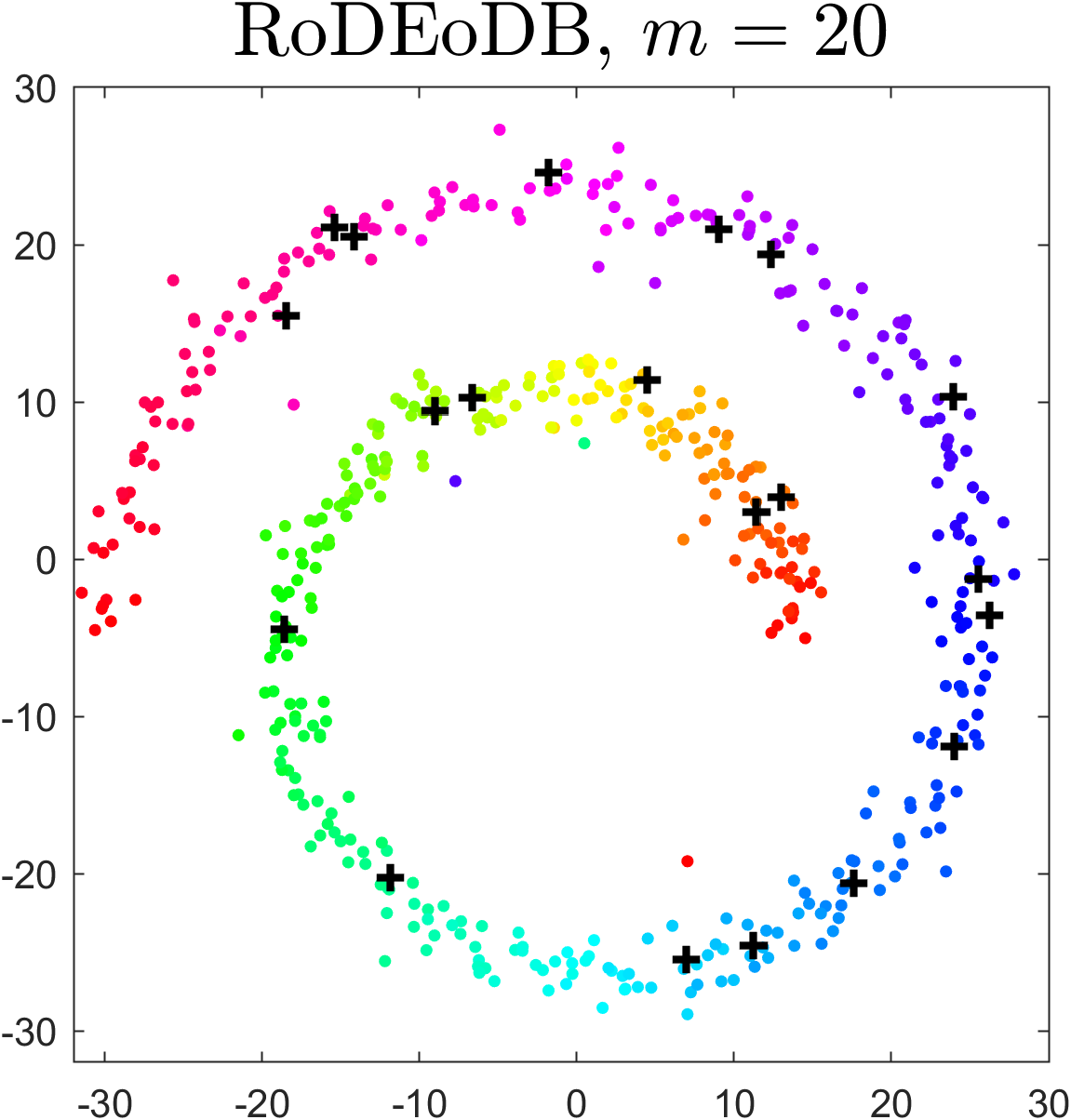}
      
    \end{subfigure}\hfill
    \begin{subfigure}[b]{0.32\linewidth}
      \centering
      \includegraphics[width=\linewidth]{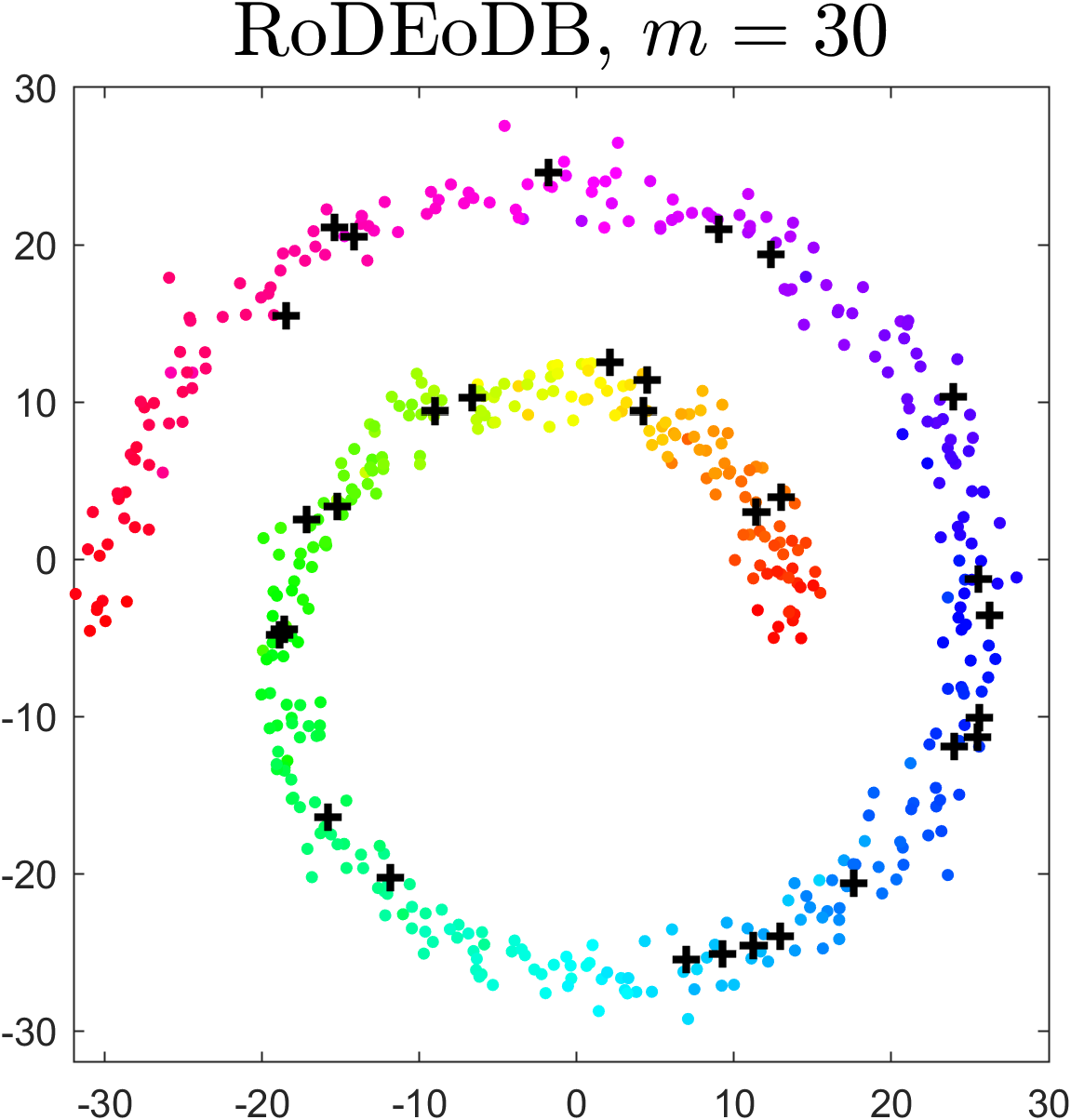}
      
    \end{subfigure}\hfill
    \begin{subfigure}[b]{0.32\linewidth}
      \centering
      \includegraphics[width=\linewidth]{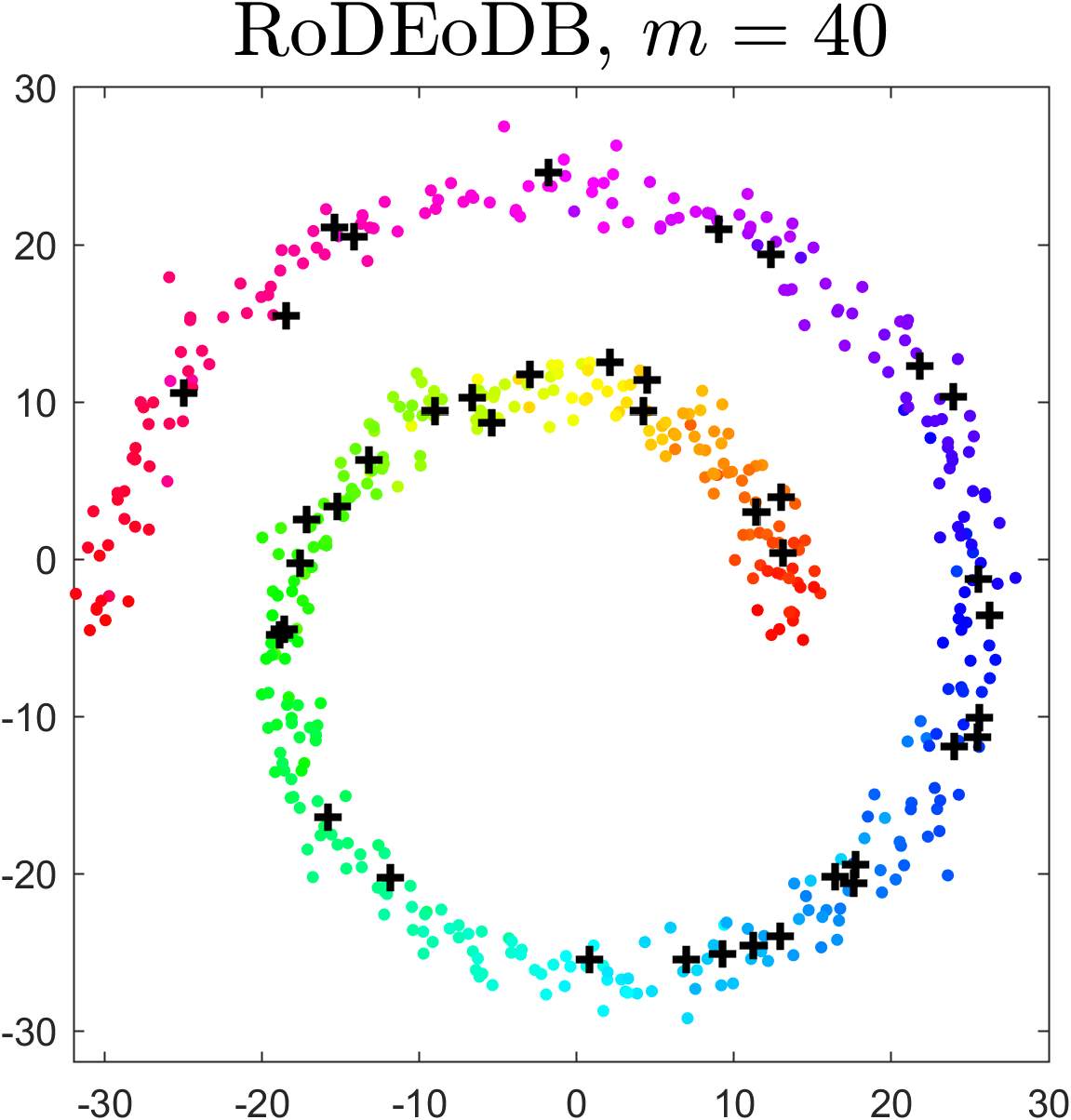}
      
    \end{subfigure}

    \vspace{1ex} 

    \begin{subfigure}[b]{0.32\linewidth}
      \centering
      \includegraphics[width=\linewidth]{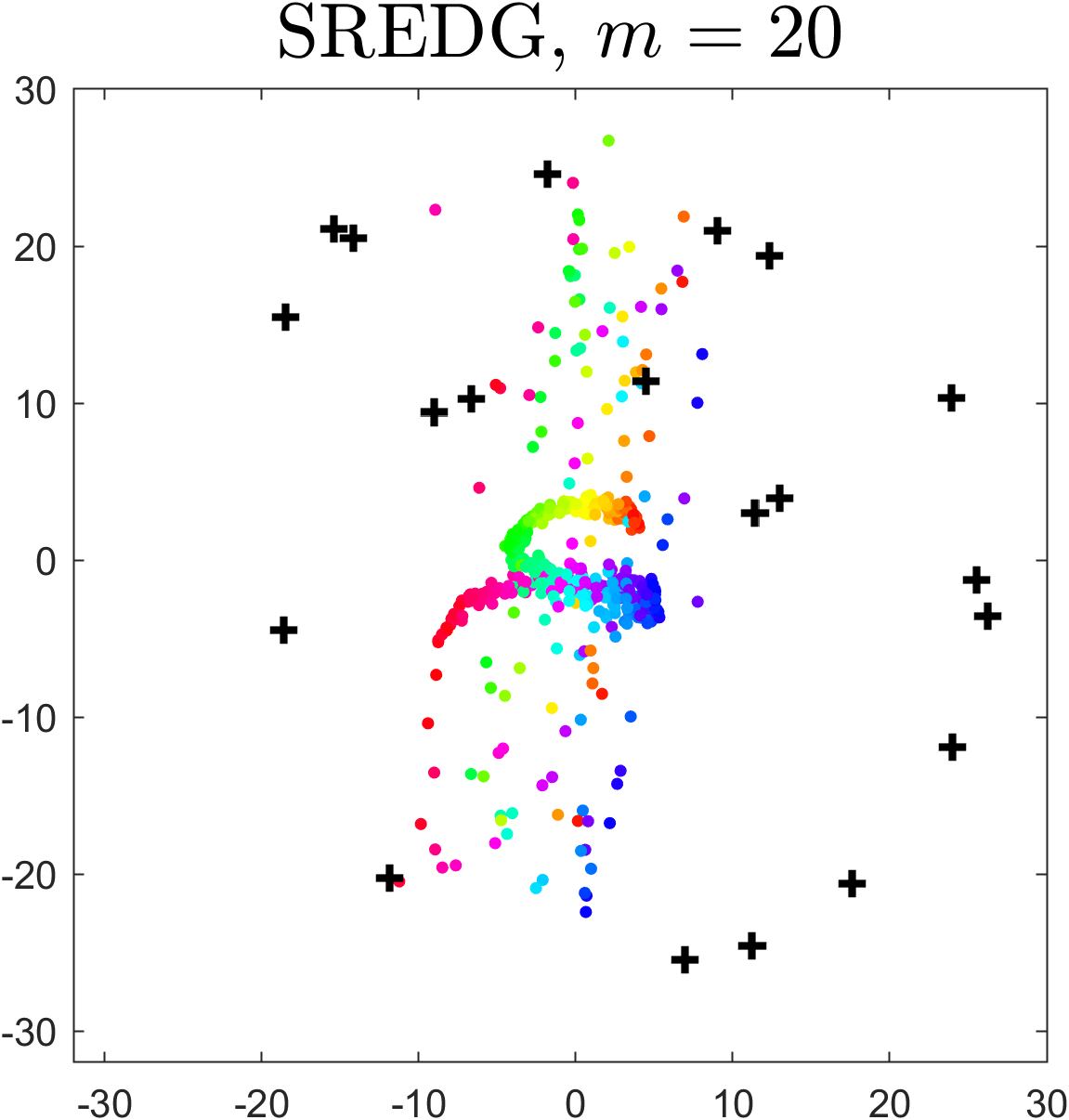}
      
    \end{subfigure}\hfill
    \begin{subfigure}[b]{0.32\linewidth}
      \centering
      \includegraphics[width=\linewidth]{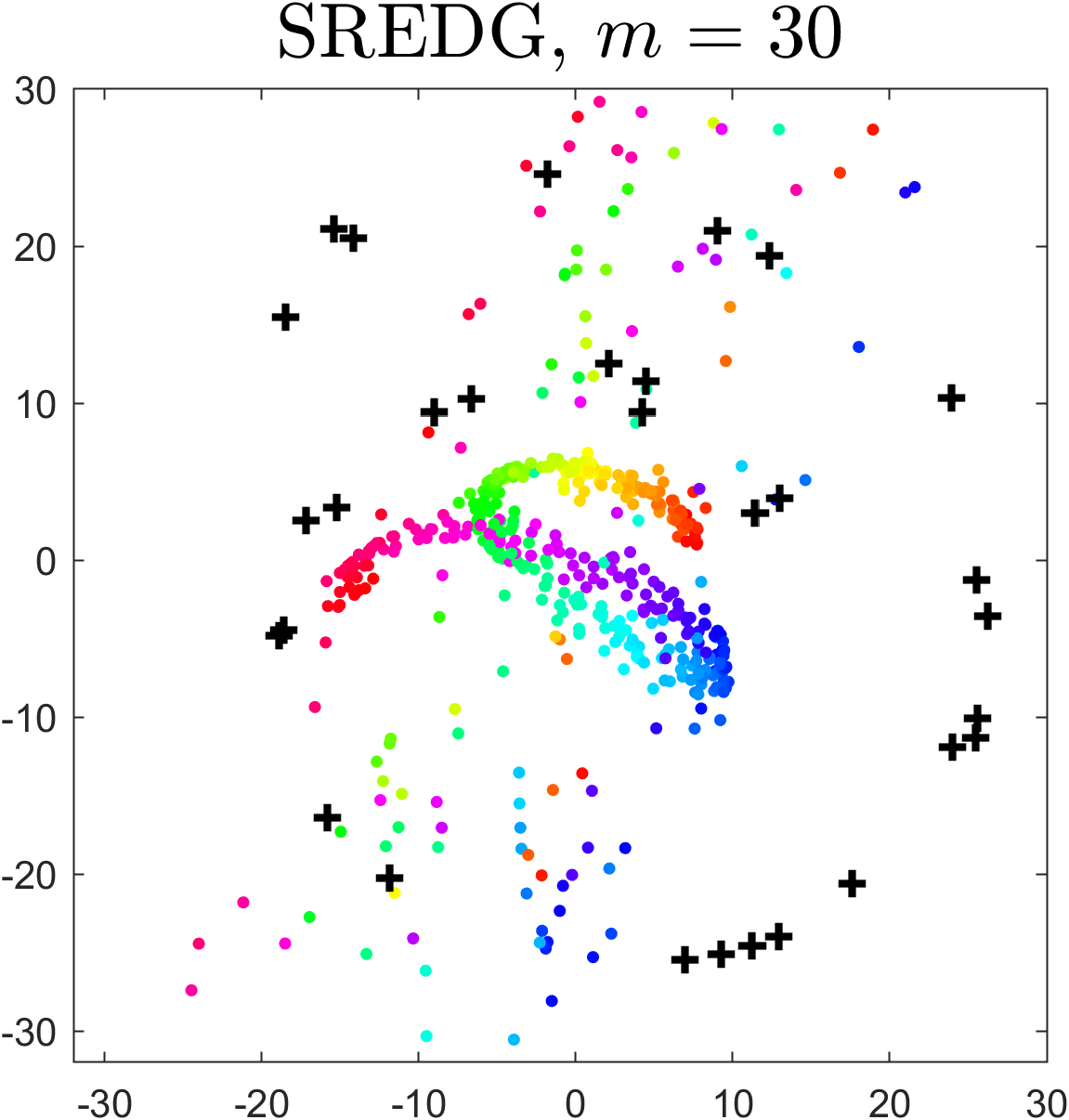}
      
    \end{subfigure}\hfill
    \begin{subfigure}[b]{0.32\linewidth}
      \centering
      \includegraphics[width=\linewidth]{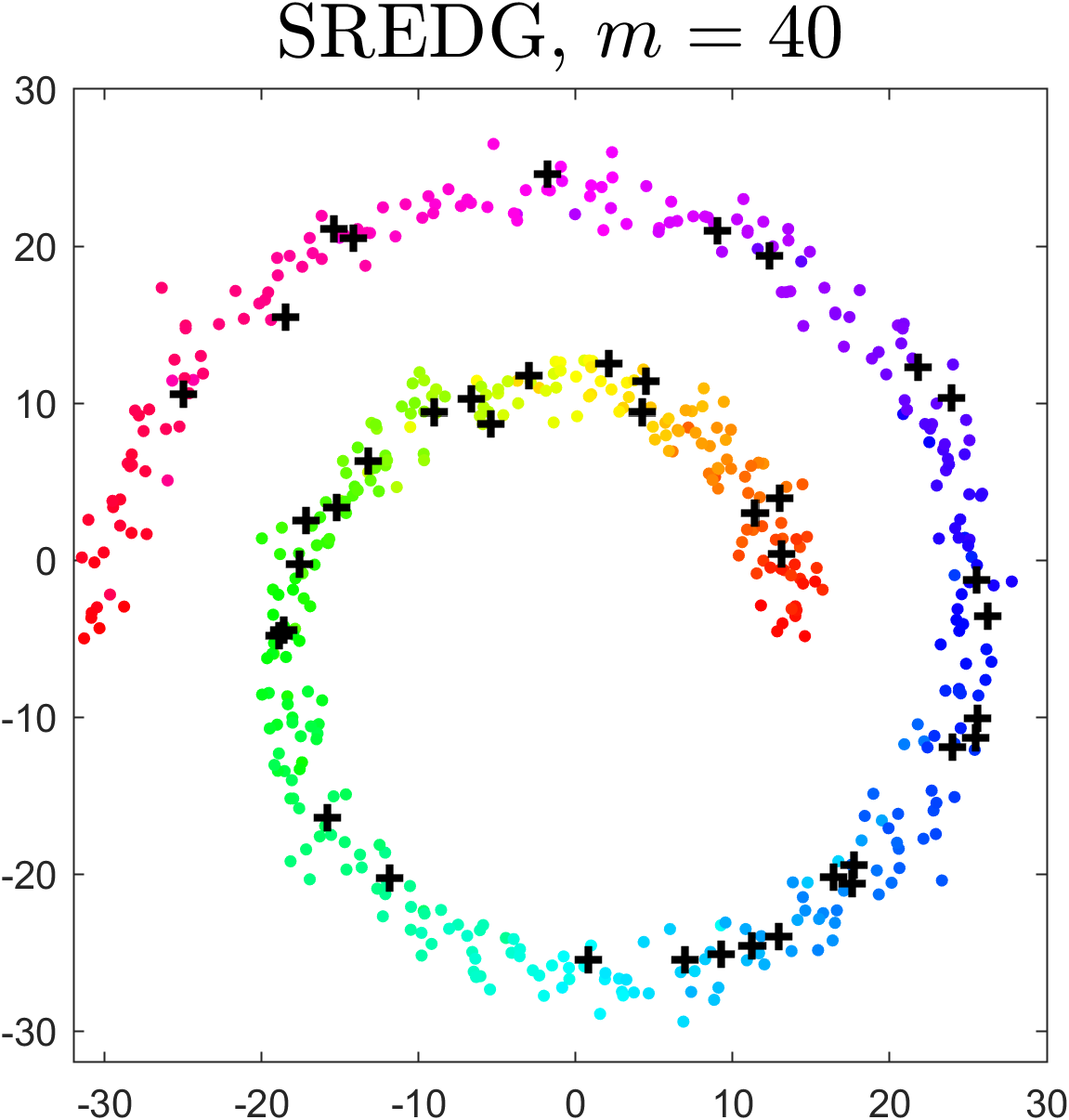}
      
    \end{subfigure}
  \end{minipage}
\vspace{-0.05in}
  \caption{\small Reconstruction of synthetic 2D spiral dataset embedded in $\mathbb{R}^{10}$ with Gaussian noise and $\alpha = 20\%$ sparse outliers. Points are colored by their original angular index $\theta$; anchor locations are highlighted by black `\textbf{+}' symbols. Left: original spiral. Right: Procrustes-aligned reconstructions obtained with RoDEoDB (top row) and SREDG (bottom row) for anchor counts $m \in \{20, 30, 40\}$.}
  \label{fig:synthetic_spiral}
  \vspace{-0.1in}
\end{figure}

\noindent\textbf{Molecular conformation.} 
We evaluate the performance of our method on a real-world molecular dataset: the 3D structure of protein 1AX8, which consists of 953 atoms and is obtained from the Protein Data Bank \cite{berman2000protein}. A random subset of $m$ atoms is selected as anchors. The distances are obtained accordingly and $\alpha$-sparse outliers are added to the anchor-target distance measurements. In protein structure prediction, sparse outliers in distance constraints frequently arise from experimental noise, signal overlap, or misassignments, especially in techniques like NMR spectroscopy.

\Cref{tab:1ax8_reconstruction} reports the reconstruction comparison between RoDEoDB and SREDG on the reconstruction of the protein structure of 1AX8. Across all values of anchor count $m$, RoDEoDB consistently outperforms SREDG both in accuracy and in recovery rate.  Notably, the recovery rate RoDEoDB escalated rapidly when the anchor count $m$ increases, while SREDG's recovery rate increases gradually at the same time. 
Furthermore, for the successfully recovered instances, RoDEoDB constantly achieves much better RMSE, which indicates higher reconstruction quality. The selected visual results by RoDEoDB are presented in \Cref{fig:protein_reconstruction}, which show visually perfect reconstruction quality for $10\%$ and $20\%$ outliers as suggested by \Cref{tab:1ax8_reconstruction}. 
With heavy $30\%$ outliers, the reconstruction exhibits slight artifacts while it still recovers the majority of the structure. 
These results highlight RoDEoDB’s robustness and effectiveness in the complicated molecular conformation task.


\begin{table}[t]
\centering
\small
\setlength{\tabcolsep}{4pt}
\vspace{-0.1in}
\caption{\small Reconstruction of protein 1AX8: comparison of RoDEoDB and SREDG over 1000 trials for varying anchor counts \(m\) at $10\%$ outliers. A lower RMSE is better and a trial is considered recovered if RMSE $\leq 1$. Same stopping criteria were used in both algorithms.}
\vspace{0.15em}
\label{tab:1ax8_reconstruction}
\begin{tabular}{ccccccccc}
\toprule
\multirow{3}{*}{$m$}
  & \multicolumn{4}{c}{RoDEoDB}
  & \multicolumn{4}{c}{SREDG} \\
\cmidrule(lr){2-5}\cmidrule(lr){6-9}
  & RMSE   & STD    & RMSE      & Recovery  
  & RMSE   & STD    & RMSE      & Recovery   \\
  & (all)  & (all)  & (recovered)     & rate (\%)
  & (all)  & (all)  & (recovered)     & rate (\%)   \\
\midrule
30
  & $3.49$    & $4.66$     & $\mathbf{5.66\times10^{-2}}$ & $\mathbf{42.30}$
  & $\mathbf{2.36}$ & $\mathbf{1.45}$     & $7.92\times10^{-1}$            & $1.30$     \\

35
  & $\mathbf{1.47}$ & $3.16$ & $\mathbf{3.49\times10^{-2}}$ & $\mathbf{68.30}$
  & $1.88$          & $\mathbf{0.79}$     & $6.21\times10^{-1}$            & $4.10$     \\

40
  & $\mathbf{5.22\times10^{-1}}$ & $1.72$ & $\mathbf{2.31\times10^{-2}}$ & $\mathbf{86.20}$
  & $1.56$          & $\mathbf{0.61}$     & $4.47\times10^{-1}$            & $11.00$    \\

45
  & $\mathbf{2.02\times10^{-1}}$ & $0.90$ & $\mathbf{7.19\times10^{-3}}$ & $\mathbf{93.10}$
  & $1.37$          & $\mathbf{0.66}$     & $3.51\times10^{-1}$            & $19.20$    \\

50
  & $\mathbf{6.49\times10^{-2}}$ & $\mathbf{0.42}$ & $\mathbf{1.43\times10^{-3}}$ & $\mathbf{97.20}$
  & $1.11$          & $0.67$     & $2.87\times10^{-1}$            & $31.60$    \\

55
  & $\mathbf{2.47\times10^{-2}}$ & $\mathbf{0.29}$ & $\mathbf{2.83\times10^{-4}}$ & $\mathbf{99.10}$
  & $8.68\times10^{-1}$          & $0.67$     & $2.62\times10^{-1}$            & $47.70$    \\

60
  & $\mathbf{3.66\times10^{-3}}$ & $\mathbf{0.08}$ & $\mathbf{2.59\times10^{-13}}$& $\mathbf{99.80}$
  & $6.43\times10^{-1}$          & $0.65$     & $2.00\times10^{-1}$            & $61.50$    \\

65
  & $\mathbf{1.86\times10^{-3}}$ & $\mathbf{0.06}$ & $\mathbf{2.50\times10^{-13}}$& $\mathbf{99.90}$
  & $4.38\times10^{-1}$          & $0.58$     & $1.61\times10^{-1}$            & $76.50$    \\

70
  & $\mathbf{2.48\times10^{-13}}$& $\mathbf{0.00}$ & $\mathbf{2.48\times10^{-13}}$& $\mathbf{100.00}$
  & $3.23\times10^{-1}$          & $0.51$     & $1.44\times10^{-1}$            & $84.20$    \\
\bottomrule
\end{tabular}
\vspace{-0.15in}
\end{table}

\begin{figure}[h]
\vspace{-0.12in}
  \centering
  \makebox[\textwidth][c]{%
    \begin{subfigure}[t]{0.31\textwidth}
      \centering
      \includegraphics[width=\linewidth]{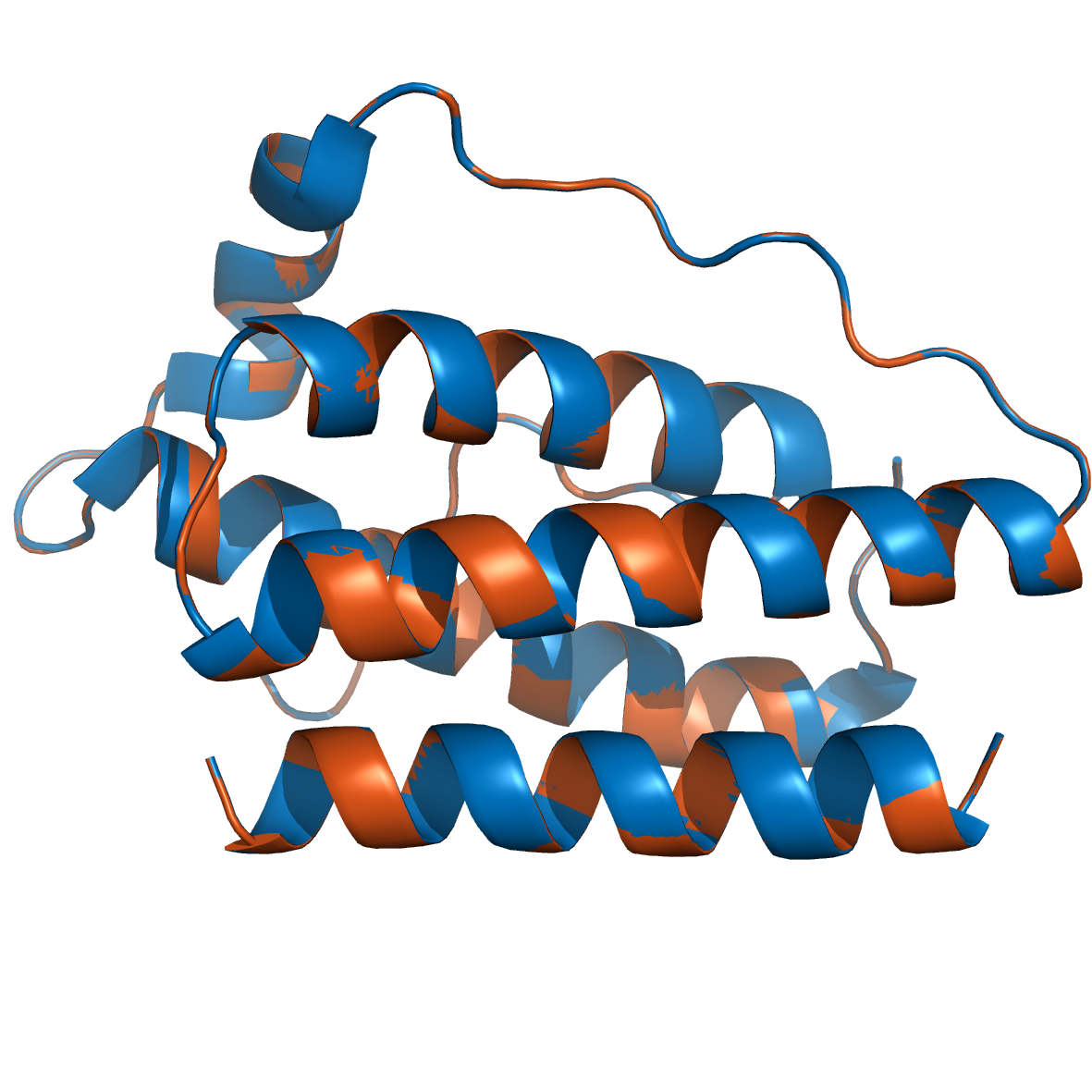}
      \vspace{-3.5em}
      \caption{$\alpha = 10\%$: RMSE = 0.1715.}
    \end{subfigure}
    \hfill
    \begin{subfigure}[t]{0.31\textwidth}
      \centering
      \includegraphics[width=\linewidth]{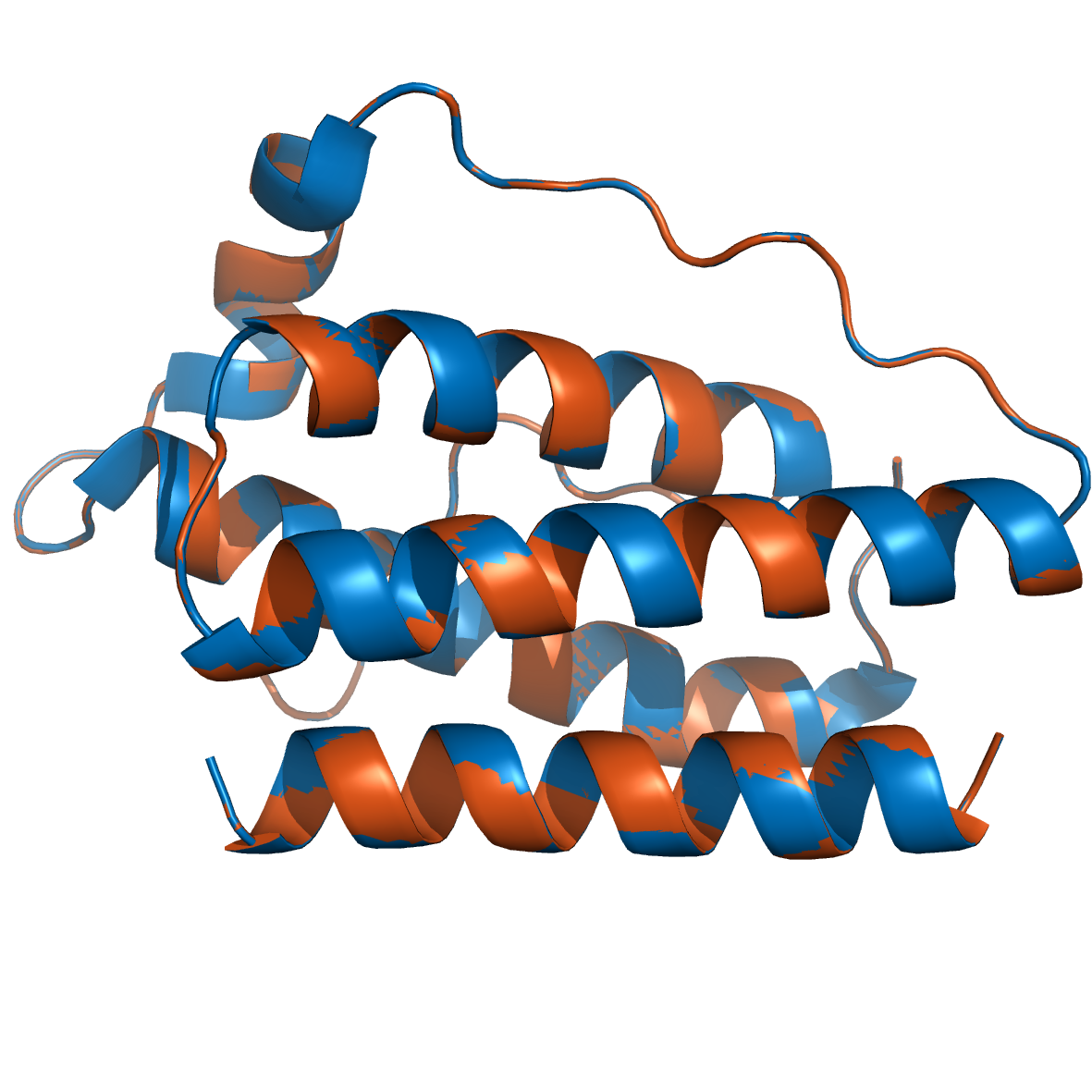}
      \vspace{-3.5em}
      \caption{$\alpha = 20\%$: RMSE = 0.8412.}
    \end{subfigure}
    \hfill
    \begin{subfigure}[t]{0.31\textwidth}
      \centering
      \includegraphics[width=\linewidth]{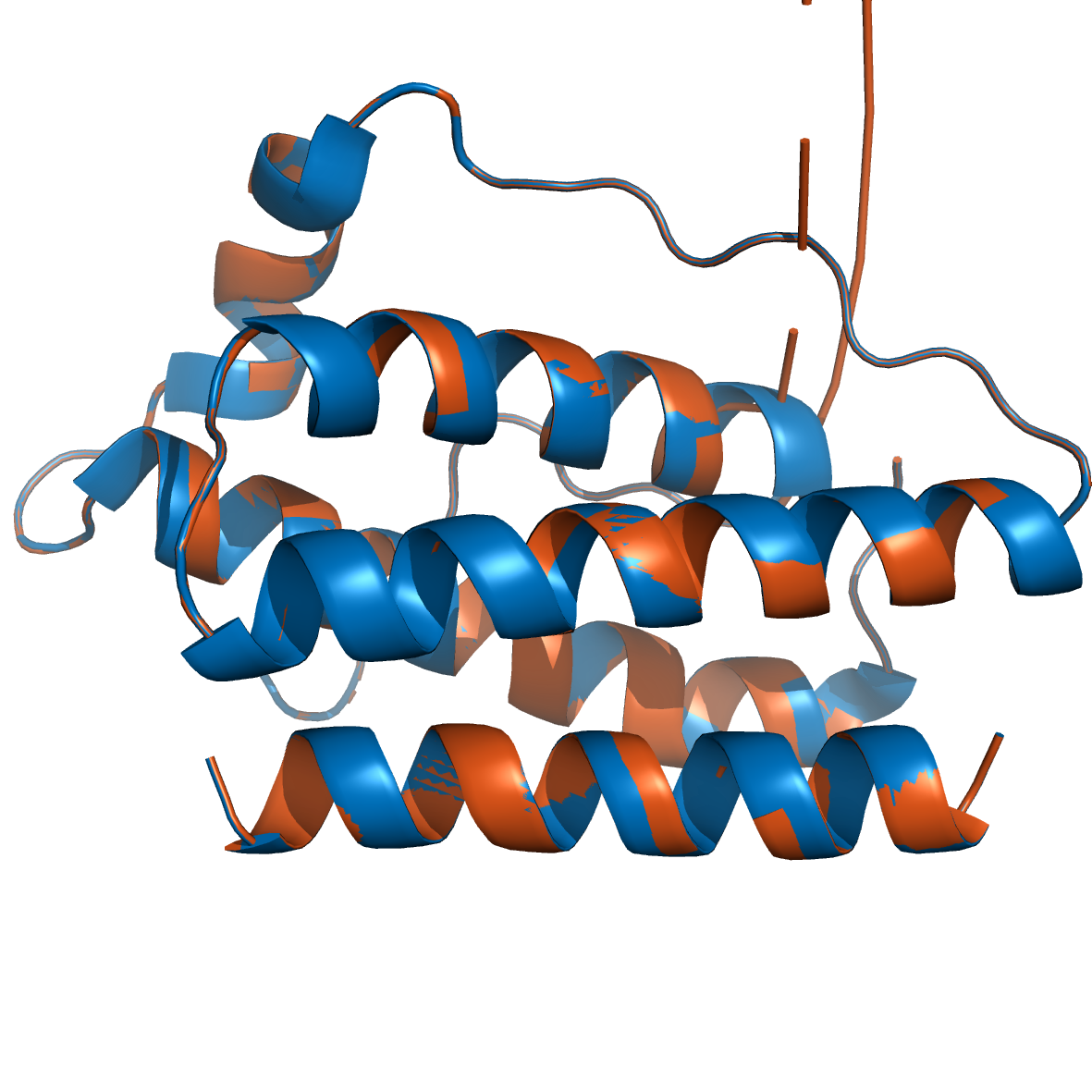}
      \vspace{-3.5em}
      \label{fig:protein_reconstruction3}
      \caption{$\alpha = 30\%$: RMSE = 3.1117.}
    \end{subfigure}
  }
  \vspace{-0.34in}
  \caption{\small Visual representation of reconstruction of Protein 1AX8 by RoDEoDB under different corruption levels. 30 anchors are used in all panels. The reconstructed protein structure (in brown) closely aligns with the true protein structure (in blue). Visualizations rendered using PyMOL~\cite{pymol}.}
  \label{fig:protein_reconstruction}
  \vspace{-0.08in}
\end{figure}

\vspace{-0.12in}
\section{Conclusion}
\vspace{-0.15in}

This paper addresses the problem of recovering the configuration of points from partial distance information, focusing on a setting where distances are observed only between a set of known anchor nodes and the remaining target nodes. 
Within this framework, we consider the challenging task of robust recovery in the presence of sparse outliers in anchor–target distance measurements. We propose a novel method based on a non-orthogonal dual basis formulation, which establishes a principled connection between the structured distance matrix and a similarly structured underlying Gram matrix. Our algorithm operates efficiently relying only on the Gram matrix blocks and achieves robust recovery through a non-convex, computationally tractable approach. Theoretical analysis guarantees exact recovery of configuration of
points under mild assumptions.
Extensive experiments on synthetic sensor location, spiral manifold recovery, and real-world protein structure data demonstrate the effectiveness of our method, RoDEoDB, which consistently outperforms existing approaches, achieving high recovery accuracy with fewer anchors and under greater levels of corruption.
Future directions include developing strategies for adaptive or domain-guided anchor selection, and analyzing their theoretical and empirical impact on robustness. Another promising line of work is to relax the assumption that all measurements from the central node (or base station) are exact. Preliminary results indicate that our framework can be extended to this setting, but a full theoretical analysis and further empirical evaluation are left for future work.


\bibliographystyle{unsrt}
\bibliography{ref}

\include{supplementary}

\end{document}

%% file: supplementary.tex
\newpage
\appendix

\begin{center}
    \LARGE
      Supplementary Material for\\
      \textbf{A Dual Basis Approach for\\ Structured Robust Euclidean Distance Geometry}
\end{center}
    
\section{Proofs} \label{sec:proofs}
In this section, we provide the mathematical proofs for the claimed theoretical results.
We start with technical lemmas, including proofs of spectral properties, condition number and incoherence properties, and perturbation analysis. These lemmas establish foundational results necessary for proving our main theorems.
Throughout, $\sigma_{\min}(\M)$ and $\sigma_{\max}(\M)$ denote the smallest and largest \emph{nonzero} singular values of $\M$, respectively; for real symmetric positive-semidefinite matrices, these values coincide with the smallest and largest nonzero eigenvalues.

    \subsection{Technical lemmas}
     \begin{lemma} \label{lemma:norm_of_J}
        Let $\I$ be the $T\times T$ identity matrix, $\1\in\mathbb{R}^T$ be the vector of all ones, and $\s\in\mathbb{R}^T$ be a column vector satisfying $\1^\top \s=1.$ Then, $\|\I-\1\s^\top\|_2 = \sqrt{T}\|\s\|_2.$
    \end{lemma}
    
    \begin{proof}
        First, consider the matrix
        \begin{align*}
            \H & = (\I -\1\s^\top)^\top(\I-\1\s^\top)                 \\
               & = \I - \s\1^\top - \1\s^\top + \s\1^\top \1 \s^\top.
        \end{align*}
        Let $\x \in \mathbb{R}^T$ be any vector orthogonal to both $\1$ and $\s$, i.e., $\1^\top \x = 0$ and $\s^\top \x = 0$. Then,
        \begin{align*}
            \H\x & = (\I - \s\1^\top - \1\s^\top + \s\1^\top \1\s^\top)\x    \\
                 & = \x - \s\1^\top\x - \1\s^\top\x + \s\1^\top \1 \s^\top\x \\
                 & = \x.
        \end{align*}
        Thus, $\x$ is an eigenvector of $\H$ corresponding to the eigenvalue $1$. Since the subspace orthogonal to $\operatorname{span}(\1, \s)$  (i.e., the space spanned by the vectors $\1$ and $\s$) has dimension at least $T - 2$, the eigenvalue $1$ has multiplicity at least $T - 2$.
    
        \textbf{Case 1: $\1$ and $\s$ are linearly independent.}
    
        We compute the effect of $\H$ on $\1$:
        \begin{align*}
            \H\1 & = \1 - \s\1^\top\1 - \1\s^\top\1 + \s\1^\top \1 \s^\top\1 \\
                 & = \1 - T \s - \1 + T \s                                   \\
                 & = \0.
        \end{align*}
        Thus, $\1$ is an eigenvector of $\H$ corresponding to the eigenvalue $0$.
    
        Next, consider the vector $\v = \s - \frac{1}{T}\1$. Applying $\H$ to $\v$, we obtain
        \begin{align*}
            \H\v & = \H \left(\s - \frac{1}{T}\1\right)                           \\
                 & = \H\s + \frac{1}{T}\H\1                                       \\
                 & = \s - \s\1^\top\s - \1\s^\top\s + \s\1^\top \1 \s^\top\s + \0 \\
                 & = - \1\|\s\|_2^2 + T \|\s\|_2^2 \s                             \\
                 & = T \|\s\|_2^2 \left(\s - \frac{1}{T}\1 \right)                \\
                 & = T \|\s\|_2^2 \v.
        \end{align*}
    
        This implies that $\v$ is an eigenvector of $\H$ with eigenvalue $T\|\s\|_2^2$.
    
        Therefore, the eigenvalues of $\H$ are $0$, $1$ (with multiplicity at least $T - 2$), and $T \|\s\|_2^2$. Since
        \begin{align*}
            1 = \|\1\tr \s\|_2^2 \leq \|\1\|_2^2\|\s\|_2^2 = T\|\s\|_2^2,
        \end{align*}
        the largest eigenvalue of $\H$ is $T\|\s\|_2^2$. Therefore, $\|\I - \1\s^\top\|_2 = \sqrt{T}\|\s\|_2$.

        \textbf{Case 2: $\1$ and $\s$ are linearly dependent.}
    
        When $\1$ and $\s$ are linearly dependent, i.e.  $\s = \frac{1}{T}\1$, this is a special case of the above where $\|\s\|_2 = \frac{1}{\sqrt{T}}$. Therefore, $\|\I - \1\s^\top\|_2 = \sqrt{T}\|\s\|_2$.
    \end{proof}

    \begin{corollary} \label{cor:norm_of_J}
        Suppose that the set of anchor indices $\cI \subseteq [T]$ is uniformly sampled without replacement, and $m = |\cI|$. Define the set of target indices as $\cJ = [T] \setminus \cI$ with $n = |\cJ| = T - m$. Define the centering matrix $\J = \I - \1 \s^\top$, where $
		[\s]_i =
		\begin{cases}
			\frac{1}{m}, & \text{for } i \in \cI;   \\
			0,           & \text{for } i \in \cJ
		\end{cases}$ and $\1 \in \mathbb{R}^T$ be the vector of all ones. Then, 
        \[
            \|\J\|_2 = \sqrt{\frac{T}{m}}.
        \]
    \end{corollary}
    \begin{proof}
        By the definition of $\s$, we compute its squared norm:
        \begin{gather*}
            \|\s\|_2^2 = \sum_{i=1}^{m} \left(\frac{1}{m}\right)^2 = \frac{m}{m^2} = \frac{1}{m}.
        \end{gather*}
        From \Cref{lemma:norm_of_J}, we have:
        \[
            \|\I - \1\s^\top\|_2 = \sqrt{T} \|\s\|_2.
        \]
        Substituting $\|\s\|_2 = \frac{1}{\sqrt{m}}$, we obtain:
        \[
            \|\J\|_2 = \sqrt{T} \cdot \frac{1}{\sqrt{m}} = \sqrt{\frac{T}{m}}.
        \]
        This completes the proof.
    \end{proof}

    \subsubsection{Proof of spectral properties}
        \begin{lemma} \label{lemma:sigma_min_X}
        Let $\D\in\mathbb{R}^{T\times T}$ be an EDM, and $\X=-\frac{1}{2} \J \D \J\tr$, where  $\J$ is defined as in \Cref{thm:boundX}. Then, the following inequality holds:
        \[
            \sigma_{\min}(\X) \geq \frac{1}{2} \sigma_{\min}(\D).
        \]
    \end{lemma}
    \begin{proof}  
             
        Here,
        \begin{align*}
            \sigma_{\min}(\X) & = \sigma_{\min}\left(-\frac{1}{2} \J \D \J\tr\right)                           \\
                              & = \frac{1}{2} \sigma_{\min}(\J \D \J\tr)                            \\
                              & = \frac{1}{2} \frac{1}{\|(\J \D \J\tr)^\dagger\|_2}                 \\
                              & = \frac{1}{2} \frac{1}{\|\J^\dagger \D^\dagger (\J\tr)^\dagger\|_2} \\
                              & \geq \frac{1}{2} \frac{1}{\|\J^\dagger\|_2^2 \|\D^\dagger\|_2 }     \\
                              & = \frac{1}{2} \sigma_{\min}(\D) \sigma_{\min}(\J)^2.
        \end{align*}
        Since, $\sigma_{\min}(\J) = 1$, it follows that $    
            \sigma_{\min}(\X) \geq \frac{1}{2} \sigma_{\min}(\D).$
    \end{proof}

    \begin{lemma} \label{lemma:sigma_min_x_upper_bound}
        Let $\D\in\mathbb{R}^{T\times T}$ be an EDM, and $\X=-\frac{1}{2} \J \D \J\tr$, where  $\J$ is defined as in \Cref{thm:boundX}. Then, the following inequality holds:
        \[
            \sigma_{\min}(\X) \leq \frac{1}{2} \sigma_{\max}(\D).
        \]
    \end{lemma}
    \begin{proof}
        Since $\J = \I - \mathbf{1} \s^\top$ and $\s$ has nonzero entries only on the anchor indices $\cI$, it follows that for any target index $j \in \cJ$, we have $[\s]_j = 0$. Therefore, 
        \[
        \J^\top \e_j = \e_j.
        \]

        Let $i, j \in \cJ$ be distinct target indices, and define the unit vector $\v = \frac{1}{\sqrt{2}} (\e_i - \e_j)$. Then, $\J\tr v = (\I - s \1\tr ) \v = \v$. Now, we compute:
        \begin{align*}
        \v^\top \X \v 
        &= -\frac{1}{2} \v^\top \J \D \J^\top \v \\
        &= -\frac{1}{2} (\J^\top \v)^\top \D \J\tr\v \\ 
        &= -\frac{1}{2} \v^\top \D \v \\
        &= -\frac{1}{4} (\e_i - \e_j)^\top \D (\e_i - \e_j) \\
        &= - \frac{1}{4} \left( e_{i}\tr \D  e_{i} - 2 e_{i}\tr \D e_{j} + e_{j}\tr \D e_{j} \right) .
        \end{align*}
        Since $\D$ is an EDM, we have $\e_q\tr \D \e_q = [\D]_{q,q} = 0$ and $e_p\tr \D \e_q = [\D]_{p,q} \geq 0$ for any $p,q$. Thus,
        \begin{align*}
        \v^\top \X \v &= \frac{1}{2} |\e_{i}\tr \D \e_{j}|  \\
        &\leq \frac{1}{2} \| \e_{i}\tr \|_{2} \| \D \|_{2} \| \e_{j} \|_{2} \\ 
        &= \frac{1}{2} \| \D \|_{2} .
        \end{align*}

        Finally, we have $ \sigma_{\min}(\X) \leq \v^\top \X \v \leq \frac{1}{2} \sigma_{\max}(\D).$ 
    \end{proof}
    
    \begin{lemma} \label{lemma:sigma_max_X}
        Let $\D\in\mathbb{R}^{T\times T}$ be an EDM, and $\X=-\frac{1}{2} \J \D \J\tr$, where  $\J$ is defined as in \Cref{thm:boundX}. Then, the following inequality holds:
        \[
            \sigma_{\max}(\X) \leq \frac{1}{2} \sqrt{\frac{T}{m}} \sigma_{\max}(\D).
        \]
    \end{lemma}
    \begin{proof}
        Here,
        \begin{align*}
            \sigma_{\max}(\X) & = \sigma_{\max}(-\frac{1}{2} \J \D \J\tr)   \\
                              & = \frac{1}{2} \sigma_{\max}(\J \D \J\tr)    \\
                              & \leq \frac{1}{2} \|\J\|_2 \|\D\|_2 \|\J\|_2 \\
                              & = \frac{1}{2} \|\J\|_2^2 \sigma_{\max}(\D).
        \end{align*}
        From \Cref{cor:norm_of_J}, we have $\|\J\|_2 = \sqrt{\frac{T}{m}}$. Therefore,
        \[
            \sigma_{\max}(\X) \leq \frac{1}{2} \sqrt{\frac{T}{m}} \sigma_{\max}(\D).
        \]
        This finishes the proof.
    \end{proof}

    \subsubsection{Proof of condition number and incoherence properties}

    \begin{lemma} \label{lemma:kappa_X}
        Let $\D\in\mathbb{R}^{T\times T}$ be an EDM, and $\X=-\frac{1}{2} \J \D \J\tr$, where  $\J$ is defined as in \Cref{thm:boundX}. Then, the following inequality holds:
        \[
            \kappa(\X)\leq \sqrt{\frac{T}{m}} \kappa(\D).
        \]
    \end{lemma}
    
    \begin{proof}
        From \Cref{lemma:sigma_min_X} and \Cref{lemma:sigma_max_X}, we have:
        \begin{align*}
            \kappa(\X) & = \frac{\sigma_{\max}(\X)}{\sigma_{\min}(\X)}                                               \\
                       & \leq \frac{\frac{1}{2} \sqrt{\frac{T}{m}} \sigma_{\max}(\D)}{\frac{1}{2} \sigma_{\min}(\D)} \\
                       & = \sqrt{\frac{T}{m}} \kappa(\D).
        \end{align*}
    This finishes the proof.
    \end{proof}

    \begin{lemma} \label{lemma:svdrow}
        Suppose $\D \in \mathbb{R}^{T \times T}$ be a $\mu$-incohorent EDM. Let $\cI \subseteq [T]$ and $\cJ \subseteq [T]$ be the sets of indices for anchors and targets, respectively, such that $\cJ = [T] \setminus \cI$. Let $\L = \D(\cI, :)$, $\F = \L(:, \cJ)$, and $\rank(\D) = \rank(\L) = \rank(\F) = r$. If $\L = \U_{\L} \Sigma_{\L} \V_{\L}^\top$ is the singular value decomposition of $\L$, then the following inequality holds:
        \[
            \|(\V_{\L}(\cJ, :))^\dagger\|_2 = \|\V_{\D}(\cJ, :)^\dagger\|_2.
        \]
    \end{lemma}
    
    \begin{proof}
        Due to the rank-preserving nature of row selection, $\L = \D(\cI, :)$ shares the same right singular vectors as $\D$, up to an orthogonal transformation, i.e., $\V_{\L} = \V_{\D} \Q$ for some orthogonal matrix $\Q \in \mathbb{R}^{r \times r}$. Consequently, for any index set $\cJ \subseteq [T]$, we have
        \begin{align*}
            \V_{\L}(\cJ, :) & = \V_{\D} \Q (\cJ, :) \\
                            & = \V_{\D}(\cJ, :) \Q.
        \end{align*}
        Now,
        \[
            (\V_{\L}(\cJ, :))^\dagger = \Q^\top \V_{\D}(\cJ, :)^\dagger.
        \]
        Since $\Q$ is orthogonal, it follows that $\|(\V_{\L}(\cJ, :))^\dagger\|_2 = \|\V_{\D}(\cJ, :)^\dagger\|_2$.
    \end{proof}
    
    \begin{lemma}\label{eq:muF} \label{thm:incoherenceF}
        Suppose $\D \in \mathbb{R}^{T \times T}$ be a $\mu$-incohorent EDM. Let $\cI \subseteq [T]$ and $\cJ \subseteq [T]$ be the sets of indices for anchors and targets, respectively, such that $\cJ = [T] \setminus \cI$. Assume that $|\cI| = m$, $|\cJ| = n$, and $T = m + n$. Define the following:
    
        \begin{itemize}
            \item $\L = \D(\cI, :)$,
            \item $\F = \L(:, \cJ) = \D(\cI, \cJ)$,
            \item $\rank(\D) = \rank(\L) = \rank(\F) = r$,
            \item $\beta_1 \coloneq \sqrt{\frac{m}{T}} \|\U_{\D}(\cI, :)^\dagger\|_2$,
            \item $\beta_2 \coloneq \sqrt{\frac{n}{T}} \|\V_{\D}(\cJ, :)^\dagger\|_2$.
        \end{itemize}
    
        Then, the following inequalities hold:
        \begin{enumerate}
            \item $\mu_1(\F) \leq \beta_1^2 \kappa(\D)^2 \mu(\D)$,
            \item $\mu_2(\F) \leq \beta_1^2 \beta_2^2 \kappa(\D)^2 \mu(\D)^2 $,
            \item $\kappa(\F) \leq \beta_1^2 \beta_2^2 \kappa(\D)^2 \mu(\D) \sqrt{r}$.
        \end{enumerate}
    \end{lemma}
    
    \begin{proof}
        Since $\D$ is $\mu$-incoherent and for $\cI \subseteq [T]$, we have $\L = \D(\cI, :)$. Given that $\rank(\D) = \rank(\L) = r$, applying \citep[Theorem 3.1]{cai2021rcur}, we have the following inequalities:
        \begin{align}
            \mu_1(\L)  & \leq \beta_1^2 \kappa(\D)^2 \mu(\D), \label{eq:mu1_L}         \\
            \mu_2(\L)  & \leq \mu(\D), \label{eq:mu2_L}                               \\
            \kappa(\L) & \leq \beta_1 \kappa(\D) \sqrt{\mu(\D) r}. \label{eq:kappa_L}
        \end{align}
    
        Since $\F = \L(:, \cJ)$ and $\rank(\F) = \rank(\L) = r$, from \Cref{lemma:svdrow} we have:
        \[
            \beta_2 = \sqrt{\frac{n}{T}}\|\V_{\D}(\cJ, :)^\dagger\|_2 = \sqrt{\frac{n}{T}}\|\V_{\L}(\cJ, :)^\dagger\|_2.
        \]
        Applying \citep[Theorem 3.1]{cai2021rcur} to $\F$, we obtain:
        \begin{align}
            \mu_1(\F)  & \leq \mu_1(\L), \label{eq:mu1_F_L}                              \\
            \mu_2(\F)  & \leq \beta_2^2 \kappa(\L)^2 \mu_2(\L), \label{eq:mu2_F_L}         \\
            \kappa(\F) & \leq \beta_2 \kappa(\L) \sqrt{\mu_2(\L) r}. \label{eq:kappa_F_L}
        \end{align}
    
        Substituting inequalities \eqref{eq:mu1_L}--\eqref{eq:kappa_L} into \eqref{eq:mu1_F_L}--\eqref{eq:kappa_F_L}, we derive:
        \begin{align*}
            \mu_1(\F)  & \leq \beta_1^2 \kappa(\D)^2 \mu(\D),                   \\
            \mu_2(\F)  & \leq \beta_1^2 \beta_2^2 \kappa(\D)^2 \mu(\D)^2,         \\
            \kappa(\F) & \leq \beta_1^2 \beta_2^2 \kappa(\D)^2 \mu(\D) \sqrt{r}.
        \end{align*}
        This finishes the proof.
    \end{proof}

    \begin{lemma}  \label{eq:muX}
        Let $\D\in\mathbb{R}^{T\times T}$ be an EDM, and $\X=-\frac{1}{2} \J \D \J\tr$, where  $\J$ is defined as in \Cref{thm:boundX}. Let $\D = \U_{\D} \Sigmab_{\D} \V_{\D}^\top$  and $\X = \U_{\X} \Sigmab_{\X} \V_{\X}^\top$ be the compact singular value decompositions of $\D$ and $\X$, respectively. Assume that $\rank(\X) = d$. Then, the following inequality holds:
        \begin{align*}
            \mu(\X) \leq \sqrt{\frac{T}{m}} \left(\frac{d+2}{d}\right)  \mu(\D).
        \end{align*}
    \end{lemma}
    
    \begin{proof}
        Here,
        \begin{align*}
            \X & = -\frac{1}{2} \J \D \J\tr                                                              \\
               & = -\frac{1}{2} \J \U_{\D} \Sigmab_{\D} \V_{\D}^\top \J\tr                               \\
               & = -\frac{1}{2} \P \Sigmab_{\D} (\J \V_{\D})^\top, \quad \text{where $\P = \J \U_{\D}$.}
        \end{align*}
    
        Now for any vector $\y \in \Real^T$, we have:
        \begin{align*}
            \X \y & = -\frac{1}{2} \P \Sigmab_{\D} (\J \V_{\D})^\top \y,
        \end{align*}
        which implies $\Col(\X) \subseteq \Col(\P)$, where $\Col(\A)$ denotes the column space of matrix $\A$.
    
        For any nonzero eigenvalue $\lambda$ of $\X$, there exists a corresponding eigenvector $\v$ such that $\X \v = \lambda \v$. Since $\Col(\X) \subseteq \Col(\P)$, it follows that $\v \in \Col(\P)$. Consequently, the columns of $\U_{\X}$ lie in $\Col(\P)$, implying that there exists an orthogonal matrix $\Q \in \mathbb{R}^{T \times d}$ spanning $\Col(\P)$ such that
        \[
            \U_{\X} = \Q \R, \quad \text{for some matrix } \R \in \mathbb{R}^{d \times d}.
        \]
        Since $\U_{\X}$ is orthogonal, $\R$ is also an orthogonal matrix. Thus, for any standard basis vector $\e_i$, we have:
        \begin{align*}
            \|\U_{\X} \tr \e_i\|_2 = \|\R \tr \Q \tr \e_i\|_2 = \|\Q\tr \e_i\|_2.
        \end{align*}
    
        Let $\Pi_{P} = \P (\P\tr\P)^\dagger \P\tr$ be the orthogonal projection onto $\Col(\P)$. We have, $\|\Q\tr \e_i\|_2 = \|\Pi_{\P} \e_i\|_2$.
        \begin{align*}
            \|\U_{\X} \tr \e_i\|_2 & = \|\P (\P\tr\P)^\dagger \P\tr \e_i\|_2                                                     \\
                                   & \leq \|\P (\P\tr\P)^\dagger \|_2 \| \P\tr \e_i\|_2                                          \\
                                   & \leq \|\P (\P\tr\P)^\dagger \P\tr\|_2 \|(\P\tr)^\dagger\|_2 \|\J\|_2 \| \U_{\D}\tr \e_i\|_2 \\
                                   & \leq \|(\U_{\D}\tr\J\tr)^\dagger\|_2 \|\J\|_2 \| \U_{\D}\tr \e_i\|_2                        \\
                                   & \leq  \|\J\|_2 \| \U_{\D}\tr \e_i\|_2.
        \end{align*}
    
        From \Cref{cor:norm_of_J}, we have $\|\J\|_2 = \sqrt{\frac{T}{m}}$. Now from the definition of $\mu$-incoherence, we have: 
        \begin{equation*}
            \mu(\X) \leq \sqrt{\frac{T}{m}} \left(\frac{d+2}{d}\right)  \mu(\D).
        \end{equation*}
    This finishes the proof.
    \end{proof}

    \subsubsection{Proof of perturbation analysis}

    \begin{lemma}\label{lemma:L_F_hat_relation}
        Suppose that $\D$ is an EDM and $\X$ is a gram matrix with block structure as given in \Cref{eq:Gram_blocks}. The relationship between the blocks $\A$ and $\B$ of $\X$ and the blocks $\E$ and $\F$ of $\D$ is described in \Cref{eq:AfromE} and \Cref{eq:BfromEF}, respectively. Let, $\L = \begin{bmatrix}
                \A \\
                \B^\top
            \end{bmatrix}$. Then, the following inequality holds:
        \begin{align*}
            \|\L - \hL\|_2 & \leq \|\F - \hF\|_2,
        \end{align*}
        where $\hL$ and $\hF$ are the estimate of $\L$ and $\F$ respectively.
    \end{lemma}
 
    \begin{proof}
        By the problem setup, the block $\E$ of $\D$ is assumed to be noiseless, whereas the block $\F$ contains noise. The block $\A$ depends only on $\E$, so estimate of $\A$ is exact. The block $\B$ depends on both $\E$ and $\F$, so the estimate of $\B$ is not exact. Therefore, the error in $\L$ is the same as the error in $\B$. Thus, we have:
        \begin{align*}
            \|\L - \hat{\L}\|_2 & = \|\B - \hat{\B}\|_2                                                                                               \\
                                & = \frac{1}{2} \left\| \F - \hat{\F} - \frac{1}{m} \1_{m \times m} (\F - \hat{\F}) \right\|_2                        \\
                                & \leq \frac{1}{2} \left\| \F - \hat{\F} \right\|_2 + \frac{1}{2m} \left\| \1_{m \times m} (\F - \hat{\F}) \right\|_2 \\
                                & \leq \frac{1}{2} \left\| \F - \hat{\F} \right\|_2 + \frac{1}{2} \left\| \F - \hat{\F} \right\|_2                    \\
                                & \leq \left\| \F - \hat{\F} \right\|_2.
        \end{align*}
        This finishes the proof. 
    \end{proof}
    \begin{lemma}[\textnormal{[\citenum{tropp2011improved}, Lemma 3.4]}]\label{thm:boundUX}
        Let $\D\in\mathbb{R}^{T\times T}$ be a $\mu$-incoherent EDM, and $\X=-\frac{1}{2} \J \D \J\tr$, where  $\J$ is defined as in \Cref{thm:boundX}. Assume that $\rank(\X) = d$, and $\U_{\X} \in \Real^{T \times d}$ are its first $d$ right singular vectors. Suppose that the set of anchor indices $\cI \subseteq [T]$ is uniformly sampled without replacement, and $m = |\cI|$ satisfy $m \geq \gamma (d+2) \sqrt{\frac{T}{m}} \mu(\D) \log(d) $ for some $\gamma \geq 0$. Then, the following inequality holds:
        \begin{align*}
            \|\U_{\X}(\cI, :)^\dagger\|_2 \leq \sqrt{\frac{T}{(1 - \delta) m}}
        \end{align*}
        with probability at least $1 - d \left( \frac{e^{-\delta}}{(1-\delta)^{1-\delta}}  \right)^{\gamma \log(d)}$, for all $\delta \in (0, 1)$.
    \end{lemma}

    \begin{theorem} \label{thm:boundXbyF}
        Suppose that $\D$ is a $\mu$-incoherent EDM with $\rank(\D) = d + 2$ and $\X$ is its corresponding gram matrix with block structure as defined in \eqref{eq:EDM_blocks}. The relationship between the blocks $\A$ and $\B$ of $\X$ and the blocks $\E$ and $\F$ of $\D$ is described in \eqref{eq:AfromE} and \eqref{eq:BfromEF}, respectively. Suppose that the set of anchor indices $\cI \in [T]$ is uniformly sampled without replacement, and $m = |\cI|$ satisfy $m  \geq \gamma (d+2) \sqrt{\frac{T}{m}} \mu(\D) \log(d) $ for some $\gamma \geq 0$. Define the set of target indices as  $\cJ = [T] \setminus \cI$. Let, $\F = \D(\cI,\cJ)$ and $\A = \X(\cI,\cI)$. Let the compact SVD of $\X$ is $\X = \U_{\X} \Sigmab_{\X} \U_{\X}^\top$, then for any Schatten p-norm, if $ \sigma_{\min}(\A) \geq 12 \| \F-\hF \|$, then the following inequality holds:
        \begin{align*}
            \|\X - \hX\| \leq \left( \frac{7}{3} \sqrt{\frac{T}{(1 - \delta) |I|}} + \frac{25}{6} \left( \sqrt{\frac{T}{(1 - \delta) |I|}} \right)^2 + \frac{1}{6} \right) \|\F - \hF\|.
        \end{align*}
    \end{theorem}
    
    \begin{proof}
        Let, $\L = \X(\cI,:) = \begin{bmatrix}
                \A \\
                \B^\top
            \end{bmatrix}$. According to \citep[Remark 3.14]{arias2020perturbation}, we have:
        \begin{align*}
            \|\X - \hX\| & \leq \left( \frac{7}{6} \left( \|\U_{\X}(\cI,:)^\dagger \| + \|\U_{\X}(\cI,:)^\dagger \| \right) \right.                            \\
                         & \quad + \frac{25}{6} \|\U_{\X}(\cI,:)^\dagger \|\|\U_{\X}(\cI,:)^\dagger \| \left. + \frac{1}{6} \right) \|\L - \hL\|\,             \\
                         & \leq \left( \frac{7}{3} \|\U_{\X}(\cI,:)^\dagger \| + \frac{25}{6} \|\U_{\X}(\cI,:)^\dagger \|^2 + \frac{1}{6} \right) \|\L - \hL\|.
        \end{align*}
        From \Cref{lemma:L_F_hat_relation}, we have $\|\L - \hL\| \leq \|\F - \hF\|$ and substituting the value of $\|\U_{\X}(\cI,:)^\dagger \|$ from \Cref{thm:boundUX}, we obtain the desired result.
    \end{proof}    
    
    \begin{lemma}[\textnormal{[\citenum{cai2021rcur}, Lemma B.3]}] \label{lem:boundAinv}
        Let $\D\in\mathbb{R}^{T\times T}$ be a $\mu$-incoherent EDM with $\rank(\D) = d+2$, and $\X=-\frac{1}{2} \J \D \J\tr$, where  $\J$ is defined as in \Cref{thm:boundX}. Suppose that the set of anchor indices $\cI \subseteq [T]$ is uniformly sampled without replacement, and $m = |\cI|$ satisfy $m \geq \gamma (d+2) \sqrt{\frac{T}{m}} \mu(\D) \log(d) $ for some $\gamma \geq 0$. Let, $\A = \X(\cI,\cI)$. Then, the following inequality holds:
        \begin{align*}
            \|\A^\dagger\|_2 \leq \frac{T}{(1 - \delta) m \sigma_{\min}(\X)}
        \end{align*}
        with probability at least $1 - 2 d \left( \frac{e^{-\delta}}{(1-\delta)^{1-\delta}}  \right)^{\gamma \log(d)}$, for all $\delta \in (0, 1)$.
    \end{lemma}
    
    \begin{theorem}[\textnormal{[\citenum{netrapalli2014non}, Theorem 1]}]
        Let $\F \in \Real^{m \times n}$ satisfies the incoherence condition in \Cref{eq:muF}, $\rank(\F) = d+2$, and $\S \in \Real^{m \times n}$ is an $\alpha$-sparse matrix as defined in \Cref{def:sparse}. With properly chosen parameters, the output of AltProj,  $\F_k$, satisfies
        \begin{align*}
            \|\F - \F_k\|_2 \leq \varepsilon, \quad \|\S - \S_k\|_2 \leq \frac{\varepsilon}{\sqrt{m n}} \quad \text{and} \quad \supp(\S_k) \subseteq \supp(\S)
        \end{align*}
        in $k = \mathcal{O}\left(r \log \left(\frac{\|\F\|_2}{\varepsilon}\right)  \right)$ iterations.
    \end{theorem}

    \begin{lemma} \label{thm:boundXbyX}
        Given the notations and assumptions of \Cref{thm:boundX}, suppose $\F$ satisfies the incoherence condition in \Cref{eq:muF} with sparsity levels $\alpha \leq \mathcal{O}\left( \frac{1}{(\mu_1(\F)\lor\mu_2(\F))r} \right)$. Suppose, AltProj is run for sufficiently many iterations in Line 3 of \Cref{algo:RoDEoDB} such that $\| \F - \hF \|_2 \leq \varepsilon \frac{(1 - \delta) \sigma_{\min}(\X) }{12} \frac{m}{T}$, then the following inequality holds:
        \begin{align*}
            \|\X - \hX\|_2 \leq \frac{5}{9} \varepsilon \sigma_{\min}(\X), 
        \end{align*}
        with probability at least $1 -  \frac{2d}{T^{c(\delta + (1-\delta)\log(1-\delta))}}$.
    \end{lemma}
    \begin{proof}
        From \Cref{lem:boundAinv}, we have:
        \begin{align*}
            \sigma_{\min} (\A) & \geq (1 - \delta) \sigma_{\min}(\X) \frac{m}{T}                         \\
                               & > 12 \varepsilon \frac{(1 - \delta) \sigma_{\min}(\X) }{12} \frac{m}{T} \\
                               & \geq 12 \| \F - \hF \|_2.
        \end{align*}
    
        Now, from \Cref{thm:boundX}, we have:
        \begin{align*}
             & \quad~\|\X - \hX\|_2                                                                                                                                                                                                      \\
             & \leq \left( \frac{7}{3} \sqrt{\frac{T}{(1 - \delta) m}} + \frac{25}{6} \left( \sqrt{\frac{T}{(1 - \delta) m}} \right)^2 + \frac{1}{6} \right) \varepsilon \frac{(1 - \delta) \sigma_{\min}(\X) }{12} \frac{m}{T} \\
             & \leq \left( \frac{7}{36} \sqrt{\frac{T}{(1 - \delta) m}} \left(\frac{(1 - \delta)m}{T}\right) + \frac{25}{72} \left( \frac{T}{(1 - \delta) m} \right) \left(\frac{(1 - \delta)m}{T}\right) \right.          \\
             & \qquad \left. + \frac{1}{72} \left(\frac{(1 - \delta)|\cI|}{T}\right) \right) \varepsilon \sigma_{\min}(\X)                                                                                                              \\
             & \leq \left( \frac{7}{36}  \sqrt{\frac{(1 - \delta) m}{T}} + \frac{25}{72} + \frac{1}{72} \left(\frac{(1 - \delta)m}{T}\right) \right) \varepsilon \sigma_{\min}(\X)                                                \\
             & \leq \frac{5}{9} \varepsilon \sigma_{\min}(\X).
        \end{align*}
    This finishes the proof.
    \end{proof}

   \subsection{Proof of main results}

    We now have all the ingredients to prove \Cref{thm:boundX,thm:boundP}.
    \begin{proof}[Proof of \Cref{thm:boundX}]
        From \Cref{thm:boundXbyX}, we have
        \begin{align*}
            \|\X - \hX\|_2 &\leq \frac{5}{9} \varepsilon \sigma_{\min}(\X).
        \end{align*}
        Now,
        \begin{align*}
            \frac{\|\X - \hX\|_2}{\|\X\|_2} &\leq \frac{5}{9} \varepsilon \frac{\sigma_{\min}(\X)}{\sigma_{\max}(\X)} \\ 
            &\leq \frac{5}{9} \varepsilon \kappa(\X)^{-1} \\ 
            &\leq \frac{5}{9} \sqrt{\frac{m}{T}}\varepsilon \kappa(\D)^{-1} \\
            &\leq \varepsilon \kappa(\D)^{-1}.
        \end{align*}
    This finishes the proof.
    \end{proof}

    \begin{proof}[Proof of \Cref{thm:boundP}]
        According to \cite[Theorem 1]{arias2020perturbation}, we have: 
        \begin{align} \label{eq:boundPbyX}
            \min_{\Q\in\mathcal{Q}}\|\P-\Q\hat{\P}\|_2\leq \| \P^\dagger \|_{2} \|\X - \hX\|_2 + d^{\frac{1}{4}} \sqrt{ \|\X - \hX\|_{2}}   .          
        \end{align}

        Now, from \Cref{thm:boundXbyX} and \Cref{lemma:sigma_min_x_upper_bound}, we have:
        \begin{align*}
            \|\X - \hX\|_2 &\leq \frac{5}{9} \varepsilon \sigma_{\min}(\X) \leq \frac{5}{18} \varepsilon \sigma_{\max}(\D)  .
        \end{align*}

        The proof follows by substituting the above bound in \eqref{eq:boundPbyX} and using the fact that $\|\D\|_2 = \sigma_{\max}(\D)$.
    \end{proof}

\newpage

\section{More experimental results} \label{sec:more experiemtns}
A Matlab implementation of the proposed algorithm is available anonymously online at \url{https://anonymous.4open.science/r/RoDEoDB-26058}.

\subsection{Sensor localization}

\begin{figure}[H]
\vspace{-0.12in}
\centering
\makebox[\textwidth][c]{%
\begin{subfigure}[t]{0.48\textwidth}
    \centering
    \includegraphics[width=\linewidth]{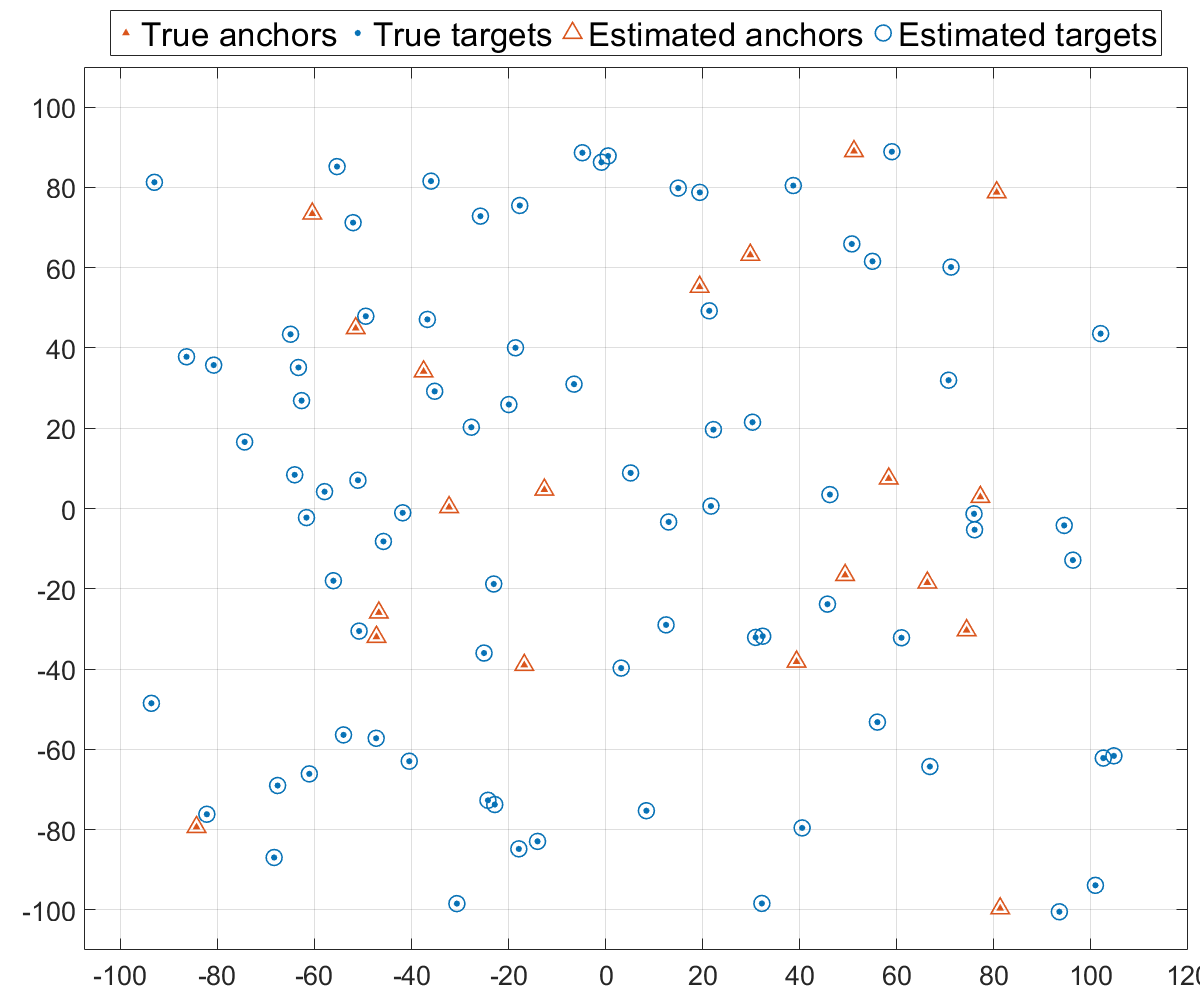}
    \caption{$\alpha = 10\%$; RMSE = $1.13 \times 10^{-13}$.}
\end{subfigure}
\hfill
\begin{subfigure}[t]{0.48\textwidth}
    \centering
    \includegraphics[width=\linewidth]{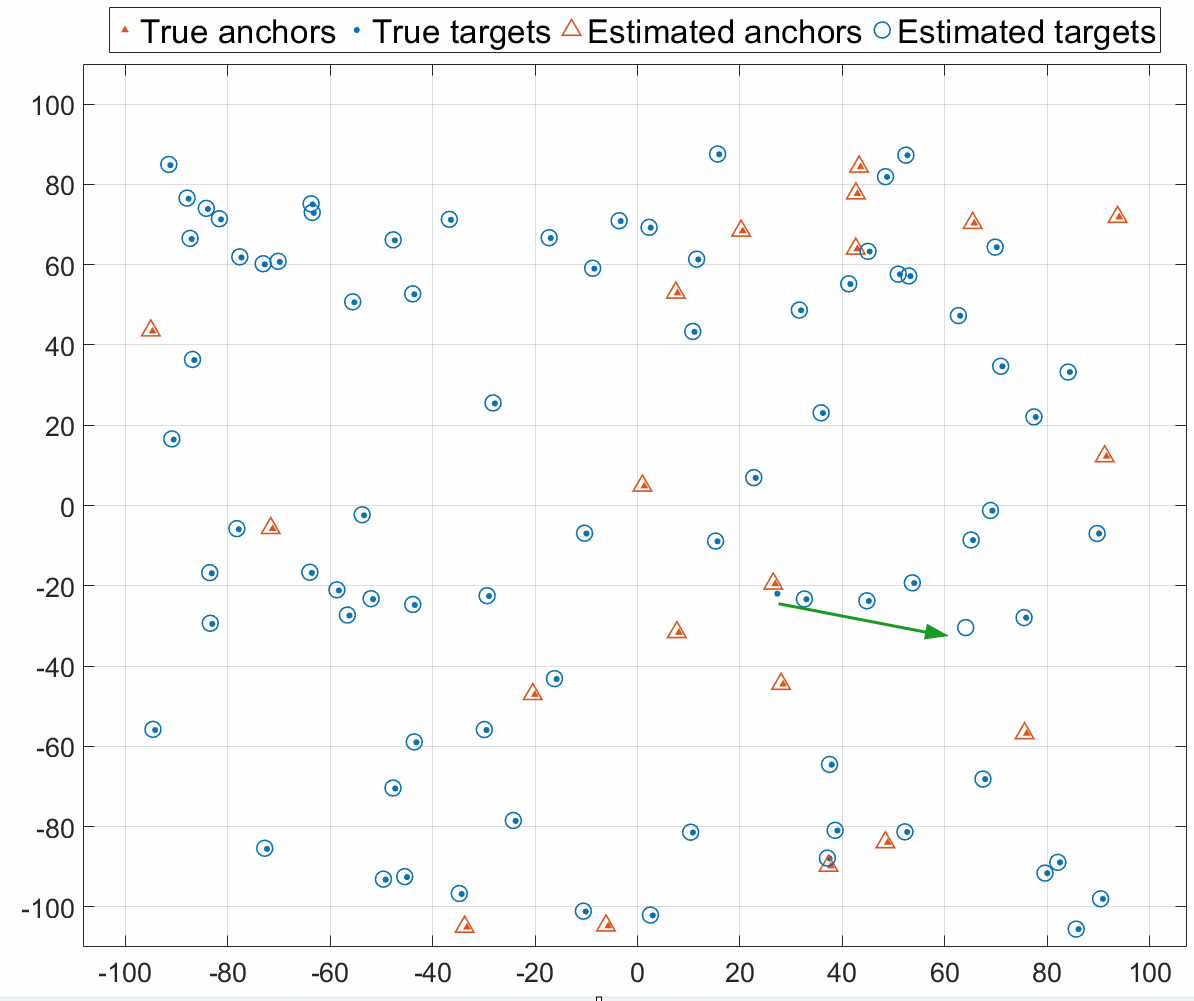}
    \caption{$\alpha = 30\%$; RMSE = $3.793$.}
\end{subfigure}

}
\caption{
Visual reconstruction results for the 2D sensor localization task with $100$ total points, $20$ anchors. Each panel shows a representative embedding recovered under different corruption scenarios.
}
\label{fig:sensor_reconstruction}
\vspace{-0.08in}
\end{figure}

\Cref{fig:sensor_reconstruction} presents visual reconstruction results for the 2D sensor localization task with $500$ total points, $20$ anchors. We consider two levels of corruption. In the left panel, corresponding to $10\%$ outliers, RoDEoDB achieves an accurate reconstruction that closely matches the ground truth. In contrast, the right panel shows the result in more challenging settings with $30\%$ outliers. Although the overall structure is preserved, some local distortions become evident, with one estimated target point (indicated by the green arrow) showing a significant deviation from its true location.

\begin{figure}[H]
\vspace{-0.12in}
\centering
\includegraphics[width=0.5\linewidth]{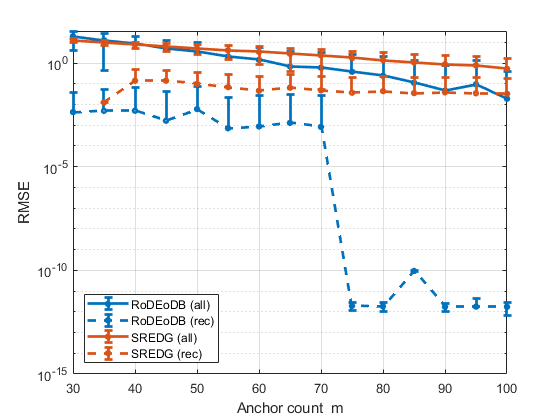}

\caption{
\small Comparison of RoDEoDB and SREDG on 3D synthetic sensor localization data with $20\%$ outliers among $500$ total points, showing RMSE versus anchor count averaged over $1000$ trials. 
}
\label{fig:sensor_reconstruction_time_rmse}
\vspace{-0.08in}
\end{figure}

In \Cref{fig:sensor_reconstruction_time_rmse}, we compare RoDEoDB and the baseline SREDG in terms of mean RMSE and average runtime over 1000 trials. RoDEoDB consistently achieves lower RMSE across all anchor counts. Notably, it exhibits a sharp improvement around $m = 70$, after which the RMSE drops to near machine precision for successfully recovered instances. In contrast, SREDG maintains a relatively high RMSE even as $m$ increases, indicating limited robustness under $20\%$ outlier corruption.

\vspace{-0.15in}
\subsection{Protein structure reconstruction}
\vspace{-0.25in}

\begin{figure}[H]
\vspace{-0.25in}
\centering
\makebox[\textwidth][c]{%
\begin{subfigure}[t]{0.48\textwidth}
    \centering
    \includegraphics[width=\linewidth]{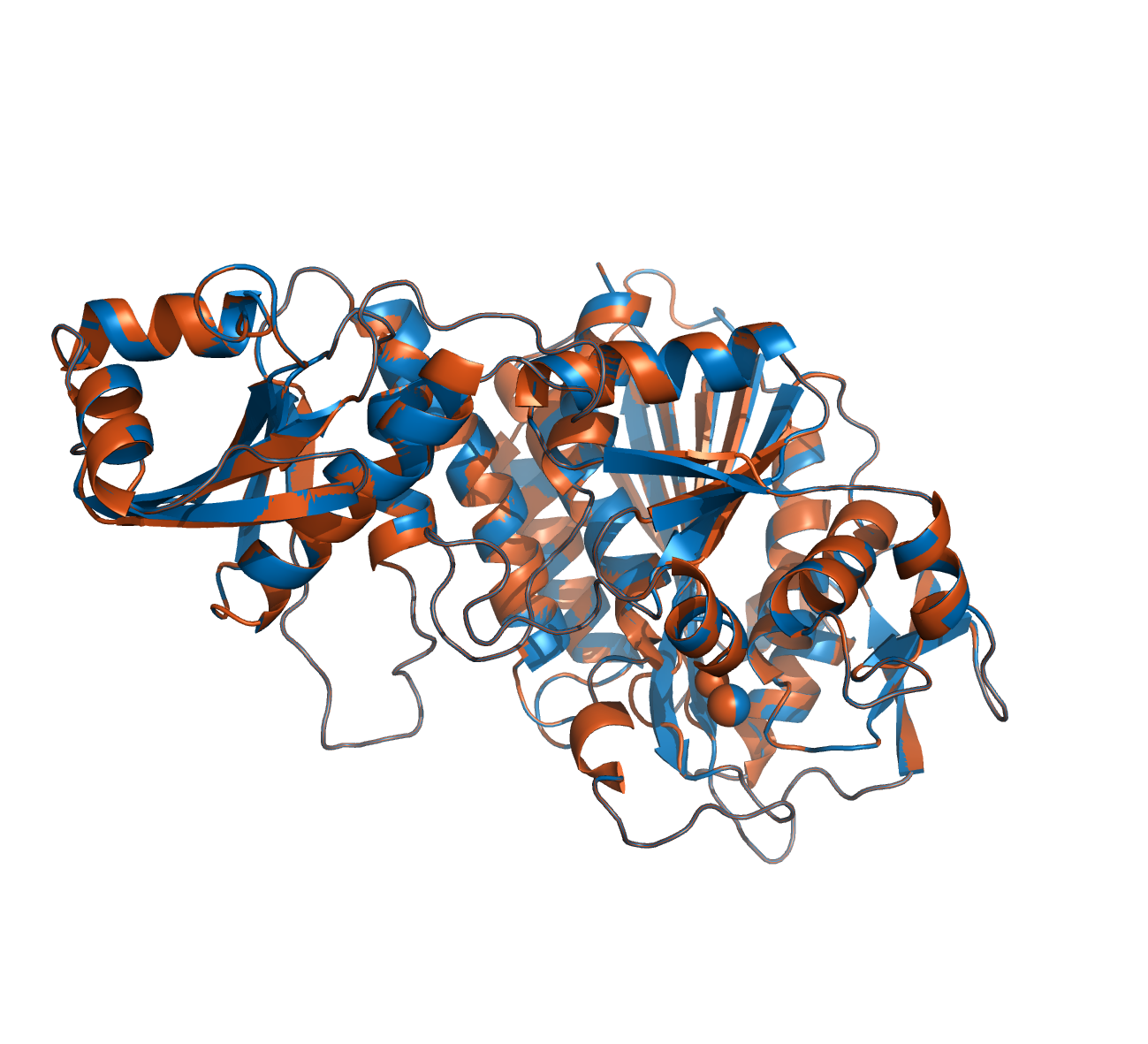}
    \vspace{-3.5em}
    \caption{$\alpha = 10\%$; RMSE = $3.39 \times 10^{-12}.$}
\end{subfigure}
\hfill
\begin{subfigure}[t]{0.48\textwidth}
    \centering
    \includegraphics[width=\linewidth]{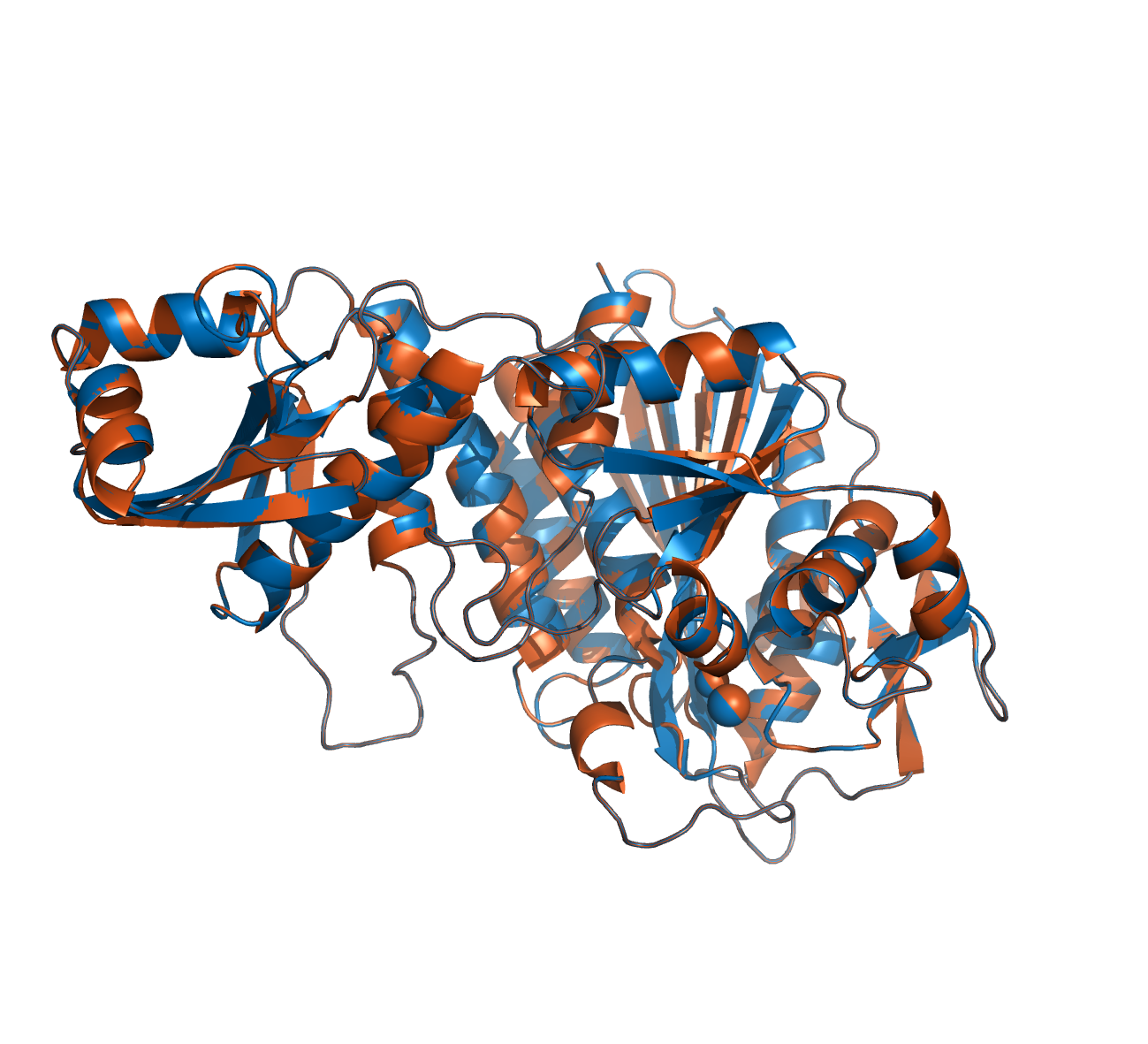}
    \vspace{-3.5em}
    \caption{$\alpha = 20\%$; RMSE = $0.91$.}
\end{subfigure}

}
\vspace{-0.3in}
\caption{\small Visual representation of reconstruction of Protein 1BPM by RoDEoDB under different corruption levels. 40 anchors are used in all panels. The reconstructed protein structure (in brown) closely aligns with the true protein structure (in blue). Visualizations rendered using PyMOL~\cite{pymol}.}
\label{fig:protein_reconstruction_ap}
\vspace{-0.08in}
\end{figure}

In \Cref{fig:protein_reconstruction_ap}, we visualize the 3D reconstruction performance of RoDEoDB on Protein 1BPM under two corruption levels. Even with $20\%$ outliers, the method closely recovers the true structure, with a minor increase in RMSE. \Cref{tab:protein_reconstruction_centered} summarizes the quantitative performance of RoDEoDB and SREDG across six different proteins under $10\%$ outlier corruption and fixed anchor count. RoDEoDB achieves consistently lower RMSE and faster runtime in most cases. These results underscore the method's accuracy, efficiency, and robustness across a variety of molecular configurations.

\begin{table}[H]
  \centering
  \small
  \renewcommand{\arraystretch}{1.2}    
  \setlength{\tabcolsep}{4pt}          
  \caption{Reconstruction of six different proteins: comparison of RoDEoDB and SREDG over 1000 trials for 50 anchor at $10\%$ outliers. For each method, RMSE ($\pm$ standard deviation) is shown, followed by runtime in seconds and recovery rate (\%). A trial is considered recovered if RMSE $\leq 1$.}
  \vspace{0.2em}
  \label{tab:protein_reconstruction_centered}
  \begin{tabular}{cccccccc}
    \toprule
    \multirow{2}{*}{Protein} & \multirow{2}{*}{Anchor}
      & \multicolumn{3}{c}{RoDEoDB}
      & \multicolumn{3}{c}{SREDG} \\
    \cmidrule(lr){3-5} \cmidrule(lr){6-8}
    & count  & RMSE & Time & Recovery 
      & RMSE & Time & Recovery \\
    & & (recovered) & (sec) & rate (\%)
      & (recovered) & (sec) & rate (\%)\\
    \midrule
    2LDJ & 314 &
      \shortstack{$\mathbf{2.84\times10^{-1}}$\\$\mathbf{(\pm0.14)}$} &
      $\mathbf{0.0499}$ & 53.7 &
      \shortstack{$7.03\times10^{-1}$\\($\pm$0.15)} &
      0.1021 & $\mathbf{63.7}$ \\ \hline
    1PTQ & 402 &
      \shortstack{$\mathbf{1.69\times10^{-3}}$\\$\mathbf{(\pm0.04)}$} &
      $\mathbf{0.0664}$ & $\mathbf{89.5}$ &
      \shortstack{$7.63\times10^{-1}$\\($\pm$0.25)} &
      0.1113 & 64.3 \\ \hline
    5WOV & 459 &
      \shortstack{$\mathbf{2.09\times10^{-3}}$\\$\mathbf{(\pm0.04)}$} &
      $\mathbf{0.0573}$ & $\mathbf{98.3}$ &
      \shortstack{$6.84\times10^{-1}$\\($\pm$0.18)} &
      0.1246 & 85.1 \\ \hline
    1UBQ & 660 &
      \shortstack{$\mathbf{2.03\times10^{-3}}$\\$\mathbf{(\pm0.04)}$} &
      $\mathbf{0.0583}$ & $\mathbf{97.8}$ &
      \shortstack{$6.77\times10^{-1}$\\($\pm$0.35)} &
      0.1357 & 47.9 \\ \hline
    2LUM & 859 &
      \shortstack{$\mathbf{5.56\times10^{-3}}$\\$\mathbf{(\pm0.07)}$} &
      $\mathbf{0.0591}$ & $\mathbf{95.8}$ &
      \shortstack{$3.37\times10^{-1}$\\($\pm$0.40)} &
      0.1416 & 57.5 \\ \hline
    1BPM & 3673 &
      \shortstack{$\mathbf{1.49\times10^{-2}}$\\$\mathbf{(\pm0.11)}$} &
      $\mathbf{0.5743}$ & $\mathbf{88.5}$ &
      \shortstack{$5.28\times10^{-2}$\\($\pm$0.16)} &
      0.7160 & 58.7 \\
    \bottomrule
  \end{tabular}
  \vspace{-0.15in}
\end{table}

\subsection{Choice of parameters}

\paragraph{RoDEoDB (Ours).}
Across all experiments, including synthetic sensor localization, spiral data, and protein structure reconstruction, we use a consistent set of parameters for RoDEoDB. The initial hard-thresholding value is set as $\xi_0 = 1.2 \cdot \max([\Fs]_{i,j})$, and the decay rate for thresholding is fixed at $\gamma = 0.95$. For alignment, we select the first row of the anchor–target block as the central nodes. Convergence is determined based on both the stability of the estimated sparse mask and the reconstruction error. Specifically, we use a convergence tolerance of $10^{-14}$, a mask stability tolerance of $10^{-3}$, and a patience parameter of 3 iterations for confirming stability. The maximum number of iterations is capped at 2000. For the protein structure visualization, we fix the number of anchors to $m = 30$, while in other experiments, we vary $m$ to study performance under different sampling conditions.

\paragraph{SREDG \citep{kundu2025structured}.}
SREDG incorporates the AccAltProj algorithm for RPCA. We use the default parameters provided by the reference implementation. The incohorence parameter $\mu$ is set to $1.1 \cdot \mu(\Fs)$, and both $\beta_0$ and $\beta$ are initialized according to the default scaling based on size of the matrix and target rank. Trimming is disabled. The convergence tolerance is set to $10^{-14}$, with a maximum iteration count of 2000. We also use the same mask stability criteria as our proposed approach RoDEoDB: a tolerance of $10^{-3}$ and a patience value of 3 to allow fair comparison.

\paragraph{GD \citep{yi2016fast}.}
We additionally compare with a baseline GD method in the sensor localization setting. The step size $\eta$ is fixed to 1.0, and the sparsification parameter is set to $\gamma = 1.1$. The maximum number of iterations is matched to RoDEoDB to ensure a fair comparison.

\section{Computational complexity}
We provide a detailed breakdown of the computational complexity of the proposed RoDEoDB, presented in \Cref{algo:RoDEoDB}, with particular emphasis on its dominant components and per-iteration costs. Let $m$ denote the number of anchor points, $n$ the number of target points, and $d$ the intrinsic dimension of points, typically with $d < m < n$.

The operator $\mathcal{A}$  and $\mathcal{B}$ perform centering along rows and columns and can be efficiently implemented and pre-computed using matrix averaging and subtraction, with total complexity $\mathcal{O}(mn)$ each. The dominant cost in \Cref{algo:RoDEoDB} arises from the iterative subroutine \textsc{Dual Basis Local Outlier Removal} (see DBAP presented in \Cref{algo:dbrpca}). Assume that DBAP is executed for $K$ iterations.

\paragraph{Cost of \textsc{DBAP}.} 
At each iteration, the most expensive step is the update of the low-rank Gram block
\[
\B^{(k+1)} = \mathcal{H}_d \mathcal{P}_{T^{(k)}} \mathcal{B}(\F - \S^{(k)}).
\]

Let $\Z := \mathcal{B}(\F - \S^{(k)})$. The projection of $\Z$ onto the tangent space of the rank-$d$ manifold at $\B^{(k)}$ is computed via:  
\begin{align*}
    \mathcal{P}_{T^{(k)}}(\Z) &= \U^{(k)} \U^{(k)\top} \Z + \Z \V^{(k)} \V^{(k)\top} - \U^{(k)} \U^{(k)\top} \Z \V^{(k)} \V^{(k)\top},
\end{align*}

where $\U^{(k)}$ and $\V^{(k)}$ are the left and right singular vectors of $\B^{(k)}$, respectively. Let
\begin{align*}
    \Q_{1}\R_{1} &= (\I - \U^{k}\U^{(k)\top}) \Z \V^{(k)}, \\
    \Q_{2}\R_{2} &= (\I - \V^{(k)}\V^{(k)\top}) \Z \U^{(k)} 
\end{align*}
be the QR decompositions of components on the right-hand side of the equations.  
Following \cite{cai2019accelerated}, the projection can be implemented as: 
\begin{align*}
    \mathcal{P}_{T^{(k)}}(\Z) &= \begin{bmatrix}
    \U^{(k)} & \Q_{2}
\end{bmatrix} \M \begin{bmatrix}
    \V^{(k)\top} \\
    \Q_{1}^\top
\end{bmatrix},
\end{align*}

where $\M = \begin{bmatrix}
    \U^{(k)\top} \Z \V^{(k)} & \R_{1}^\top \\
    \R_{2} & \0
\end{bmatrix}$ is a $2d \times 2d$ matrix. 

Now, we have the step-by-step computational complexities: 
\begin{itemize}
  \item $\Y := \Z \V^{(k)}$  \quad $\rightarrow$ $\mathcal{O}(mnd)$ flops.
  \item $\Z^\top \U^{(k)} - \V^{(k)}(\Y^\top \U^{(k)})$ \quad $\rightarrow$ $\mathcal{O}(mnd + md^2 + nd^2)$ flops.
  \item $\Y - \U^{(k)}(\U^{(k)\top} \Y)$ \quad $\rightarrow$ $\mathcal{O}(md^2 + nd^2)$ flops.
  \item Two QR factorizations \quad $\rightarrow$ $\mathcal{O}(nd^2 + md^2)$ flops.
  \item Compute the eigen-decomposition of $\M = \U_{\M} \bm{\Lambda}_{\M} \U_{\M}^\top$ \quad $\rightarrow$ $\mathcal{O}(d^3)$ flops.

  \item Here, $\U_{M(:, 1:d)}$ is the first $d$ columns of $\U_{\M}$, and now we compute $\mathcal{H}_d \mathcal{P}_{T^{(k)}}(\Z)$ as:
  \begin{itemize}
    \item $\U^{(k+1)} = [\U^{(k)}\ \Q_2] \U_{M(:, 1:d)}$ \quad $\rightarrow$ $\mathcal{O}(md^2)$ flops.   
    \item $\V^{(k+1)} = [\V^{(k)}\ \Q_1] U_{M(:, 1:d)}$ \quad $\rightarrow$ $\mathcal{O}(nd^2)$ flops.
    \item $\B^{(k+1)} = \U^{(k+1)} \bm{\Lambda}_{1:d} \V^{(k+1)\top}$ \quad $\rightarrow$ $\mathcal{O}(md^2 + mnd)$ flops.
  \end{itemize}  
\end{itemize}

Thus, each iteration of DBAP requires approximately $\mathcal{O}(mnd + md^2 + nd^2 + d^3) \approx \mathcal{O}(mnd) $ flops, since $d < m < n$. The total cost of $K$ iterations of DBAP is $\mathcal{O}(K m n d)$.

\paragraph{Post-DBAP Steps.}
After recovering $\hat{\B}$ from DBAP, the remaining steps are:

\begin{itemize}
  \item Compute $\hat{\C} = \hat{\B}^\top \A^\dagger \hat{\B}$:  
  \begin{itemize}
    \item $\A$ is fixed here, and we just solve $d$ linear systems $\A X = \hat{\B}$ \quad $\rightarrow$ $\mathcal{O}(m^2 d)$ flops.
    \item $\hat{\B}^\top X$ \quad $\rightarrow$ $\mathcal{O}(nd^2)$ flops.
  \end{itemize}
  \item From $\hat{\X} = \hat{\U}_d \hat{\Sigmab}_d \hat{\U}_d^\top$, compute $\hat{\P} = \Sigmab_d^{1/2} \hat{\U}_d^\top$ \quad $\rightarrow$ $\mathcal{O}((m+n)d^2 + nd^2)$ flops. 
\end{itemize}

Therefore, the overall complexity of RoDEoDB is $\mathcal{O}(K m n d) = \mathcal{O}(  n^2 d  \kappa(\D) \varepsilon^{-1})$ flops, by taking $K= \mathcal{O}\left(\frac{(m+n)\kappa(\D)}{m\varepsilon}\right)$ as stated in \Cref{thm:boundX}.

\section*{Broader impacts}
\vspace{-.1in}
This work presents a new algorithmic framework for solving EDG problems under sparse corruption. EDG serves as a core computational tool in many scientific and engineering domains where accurate geometric reconstruction from noisy or partial distance data is essential. While our contributions are primarily methodological, they carry potential downstream benefits across several application areas.
Potential use cases include, but are not limited to:
\begin{itemize}[leftmargin=*]
\setlength\itemsep{.1em}
  \item \textbf{Sensor network localization}: Improving robustness to faulty or adversarial distance data can enhance the reliability of systems used in environmental monitoring, autonomous navigation, and smart infrastructure.
  \item \textbf{Molecular conformation}: More accurate and stable recovery of protein structures from partial distance measurements may support research in computational biology and drug discovery.
  \item \textbf{Manifold learning and embedding}: Robust low-rank reconstructions may benefit various downstream tasks in machine learning, including dimensionality reduction, clustering, and visualization.
\end{itemize}
The anticipated positive impact lies in increasing the reliability and robustness of geometric methods used in critical sensing and modeling pipelines. The proposed dual basis formulation may also promote broader adoption of structured matrix techniques in signal processing and learning.

We are not aware of any direct negative societal impacts. While in principle, robust reconstruction algorithms could be misused (for example, in surveillance or privacy-invading applications), the method itself does not involve data collection or any domain-specific deployment. Potential misuse would depend entirely on the downstream application, not on the algorithmic contribution.